\documentclass[11pt]{article}
\usepackage[margin=1in]{geometry}
\usepackage[authoryear,round,semicolon]{natbib}

\usepackage{amsmath} % assumes amsmath package installed
\usepackage{amssymb}  % assumes amsmath package installed
\usepackage{amsthm}

\usepackage{mathtools}
\usepackage{graphicx}

 \usepackage{cite}

\usepackage{multirow}

\usepackage{tcolorbox}

\usepackage{caption}
\usepackage{booktabs}
\usepackage[table]{xcolor}

\numberwithin{equation}{section}

\newtheorem{theorem}{Theorem}[section]
\newtheorem{lemma}[theorem]{Lemma}
\newtheorem{proposition}[theorem]{Proposition}
\newtheorem{corollary}[theorem]{Corollary}

\newtheorem{assumption}[theorem]{Assumption}

\newtheorem{remark}[theorem]{Remark}

\usepackage{algorithm}
\usepackage{algorithmic}

\definecolor{darkblue}{rgb}{0.0,0.0,0.65}
\definecolor{darkred}{rgb}{0.68,0.05,0.0}
\definecolor{darkgreen}{rgb}{0.0,0.29,0.29}
\definecolor{darkpurple}{rgb}{0.47,0.09,0.29}

\usepackage[backref=page]{hyperref}
\hypersetup{
  colorlinks=true,
  citecolor=darkblue,
  linkcolor=darkred,
  filecolor=darkblue,
  urlcolor=darkblue
}

\usepackage[capitalize,noabbrev]{cleveref}

\title{\bf
Sparse Data Augmentation for Optimization with Provable Guarantees
}

\date{}

\author{Behrooz Tahmasebi\thanks{Harvard John A. Paulson School of Engineering and Applied Sciences, Harvard University, Cambridge, MA 02138, USA. Emails: \texttt{\{behrooz\_tahmasebi,mweber\}@seas.harvard.edu}} \and Melanie Weber\footnotemark[1]%
}

\begin{document}

\maketitle
% \thispagestyle{empty}
% \pagestyle{empty}
%

%%%%%%%%%%%%%%%%%%%%%%%%%%%%%%%%%%%%%%%%%%%%%%%%%%%%%%%%%%%%%%%%%%%%%%%%%%%%%%%%

\begin{abstract}
  In nonconvex optimization problems arising in geometric machine learning, data augmentation is commonly used to promote invariance by averaging empirical losses over transformations of the data. Computing the fully augmented objective, however, requires access to every element of the transformation group $G$, which may be prohibitively expensive when $G$ is large or accessible only through sampling. We study whether full augmentation can instead be approximated using a small, fixed sample of transformations acquired before optimization and reused thereafter. Under suitable regularity conditions, we show that, with probability at least $1-\delta$, gradient descent (GD) on the resulting sparsely augmented objective returns an $\varepsilon$-stationary point of the fully augmented objective using $\mathcal{O}\bigl((\log |G|+\log(1/\delta))/\varepsilon^2\bigr)$ group-transformation-oracle queries. By comparison, standard group stochastic gradient descent (group-SGD), which samples a fresh transformation at every iteration, uses $\mathcal{O}(1/\varepsilon^4)$ transformation queries. Therefore, gradient descent with fixed sparse augmentation requires fewer transformation queries than both GD applied to the fully augmented objective and group-SGD. Our proof techniques, which may be of independent interest, establish a uniform approximation of the full group-averaged gradient field by a random group average using spectral properties of group-induced operators and tools from representation theory.
  \end{abstract}

%%%%%%%%%%%%%%%%%%%%%%%%%%%%%%%%%%%%%%%%%%%%%%%%%%%%%%%%%%%%%%%%%%%%%%%%%%%%%%%%

\clearpage
\tableofcontents 
\clearpage

\section{Introduction}

We consider the classical empirical risk minimization (ERM) problem
\begin{equation}
\min_{\theta\in\mathbb{R}^p}
\mathcal{R}_n(\theta)
\coloneqq
\frac{1}{n}
\sum_{i=1}^{n}
\ell_i(x_i;\theta),
\end{equation}
where $\{x_i\}_{i=1}^n \subset \mathbb{R}^d$ denote the data (either labeled
or unlabeled), and $\ell_i(x_i;\theta)$ is the loss incurred by the parameter
$\theta$ on the $i$th sample. The parameter $\theta$ defines a predictor
$f_\theta:\mathbb{R}^d\to\mathbb{R}$, for example in a supervised learning
problem.

In geometric machine learning, one often seeks predictors that are invariant
under a known group of transformations. Let $G$ be a finite group acting on
$\mathbb{R}^d$. We are interested in predictors satisfying
\begin{equation}
f_\theta(g\cdot x)
=
f_\theta(x),
\qquad
\forall g\in G,
\quad
\forall x\in\mathbb{R}^d.
\end{equation}
Such invariances arise in many applications, including permutation invariance
for sets and graphs, sign invariance in spectral methods, translations in
images, and coordinate transformations in point clouds. A central challenge is
to incorporate these invariances into the learning problem in a principled and
computationally scalable manner.

A standard approach is \emph{data augmentation}, in which the loss of each
sample is averaged over its transformed versions. This leads to the fully
augmented empirical risk
\begin{equation}
\mathcal{R}_n^{G}(\theta)
\coloneqq
\frac{1}{n|G|}
\sum_{i=1}^{n}
\sum_{g\in G}
\ell_i(g\cdot x_i;\theta).
\end{equation}
The augmented objective incorporates the desired transformation structure into
the learning problem. However, evaluating its gradient requires averaging over
all $|G|$ transformations. Direct optimization of $\mathcal{R}_n^{G}$ can
therefore be computationally prohibitive when the group is large.

\paragraph{Streaming augmentation.}
A common alternative is to sample only a small number of transformations at
each iteration. We refer to this paradigm as \emph{streaming augmentation}.
At iteration $t$, a fresh random batch $S_t\subseteq G$ is sampled, and the
parameter is updated according to
\begin{equation}
\theta_{t+1}
=
\theta_t
-
\eta_t
\nabla\mathcal{R}_n^{S_t}(\theta_t),
\end{equation}
where $\eta_t>0$ is the step size and
\begin{equation}
\mathcal{R}_n^{S_t}(\theta)
\coloneqq
\frac{1}{n|S_t|}
\sum_{i=1}^{n}
\sum_{g\in S_t}
\ell_i(g\cdot x_i;\theta).
\end{equation}
Streaming augmentation thus performs optimization using a sequence of randomly
changing partially augmented objectives. If $|S_t|=b$ and the method runs for
$T$ iterations, it draws $bT$ transformations over the course of optimization.

To quantify this sampling requirement, we define a call to the
\emph{group-sampling oracle} as drawing an independent transformation
uniformly from $G$. The group-oracle complexity of a method is the total number
of fresh transformations drawn from $G$. This quantity is distinct from the
number of transformed-gradient evaluations: although previously sampled
transformations can be reused, the gradient over them must be recomputed at
each new iterate.

Under standard smooth nonconvex assumptions, streaming stochastic gradient
descent with $|S_t|=b$ satisfies
\begin{equation}
\mathbb{E}\left[
\min_{0\leq t\leq T}
\left\|
\nabla \mathcal{R}_n^G(\theta_t)
\right\|
\right]
=
\mathcal{O}\left(
\frac{1}{(bT)^{1/4}}
\right).
\end{equation}
Consequently, reaching an expected stationarity level of at most $\epsilon$
requires $bT=\mathcal{O}(1/\epsilon^4)$. Because the method draws $b$ fresh
transformations at each iteration, this amounts to
$\mathcal{O}(1/\epsilon^4)$ calls to the group-sampling oracle.

This motivates the central optimization question considered in this paper:
\begin{tcolorbox}
\centering
\emph{Can gradient descent on a fixed, randomly sampled augmentation achieve
an $\epsilon$-stationary point of the fully augmented objective without fresh
group-oracle queries at every iteration?}
\end{tcolorbox}

\begin{table}[t]
  \centering
  \scriptsize
  \caption{Comparison in the smooth nonconvex setting. Here,
$T(\epsilon)$ is the number of iterations required to reach an
$\epsilon$-stationary point, $b=|S_t|$ is the streaming batch size, and
$m=|S|=\mathcal{O}(\log |G|/\epsilon^2)$ is the fixed one-shot augmentation
size.}
  \label{tab:nonconvex-comparison}
  \begin{tabular*}{\linewidth}{@{\extracolsep{\fill}}llccc@{}}
  \toprule
  Method
  & Regime
  & Augmentation set
  & $T(\epsilon)$
  & Group-oracle calls \\
  \midrule
  GD on $\mathcal{R}_n^G$
  & Full
  & Full group $G$
  & $\mathcal{O}(1/\epsilon^2)$
  & $|G|$ (full group) \\
  SGD on $\mathcal{R}_n^{S_t}$
  & Streaming
  & Fresh $S_t$ at each iteration
  & $\mathcal{O}(1/(b\epsilon^4))$
  & $\mathcal{O}(1/\epsilon^4)$ \\
  GD on $\mathcal{R}_n^S$
  & One-shot
  & Fixed $S$ throughout
  & $\mathcal{O}(1/\epsilon^2)$
  & $\mathcal{O}(\log |G|/\epsilon^2)$ \\
  \bottomrule
  \end{tabular*}
  \end{table}

\paragraph{One-shot augmentation.}
We study the latter possibility through a paradigm that we call
\emph{one-shot augmentation}. Instead of drawing a fresh subset at every
iteration, we sample a single random subset $S\subseteq G$ before optimization,
fix it for the entire optimization trajectory, and run gradient descent on
\begin{equation}
\mathcal{R}_n^{S}(\theta)
\coloneqq
\frac{1}{n|S|}
\sum_{i=1}^{n}
\sum_{g\in S}
\ell_i(g\cdot x_i;\theta).
\end{equation}
The resulting iterates satisfy
\begin{equation}
\theta_{t+1}
=
\theta_t
-
\eta_t
\nabla\mathcal{R}_n^{S}(\theta_t).
\end{equation}
Conditioned on the initial draw of $S$, the objective remains fixed and the
subsequent optimization procedure is deterministic. Unlike streaming
augmentation, which continually queries the group-sampling oracle, one-shot
augmentation makes only $|S|$ oracle calls, independently of the number of
optimization iterations.

Establishing guarantees for this method presents a fundamental dependence
challenge. Gradient descent is performed on the partially augmented objective
$\mathcal{R}_n^{S}$, whereas the desired stationarity guarantee concerns the
fully augmented objective $\mathcal{R}_n^{G}$. Moreover, because the same
random subset $S$ is reused throughout optimization, every iterate $\theta_t$
depends on $S$. This differs from streaming augmentation, where the fresh batch
$S_t$ drawn at iteration $t$ is independent of $\theta_t$ conditional on the
preceding optimization history.

Pointwise concentration at a fixed parameter therefore does not control the
gradient discrepancy at the data-dependent iterates produced by the one-shot
method. Instead, we establish a new uniform approximation of the form
\begin{equation}
\sup_{\theta}
\left\|
\nabla\mathcal{R}_n^{S}(\theta)
-
\nabla\mathcal{R}_n^{G}(\theta)
\right\|
\leq \epsilon,
\end{equation}
which holds simultaneously over the parameter domain and therefore along the
entire optimization trajectory.

\paragraph{Main result.}
Under our structural and smoothness assumptions, we show that a fixed subset
of size $|S|=\mathcal{O}(\log |G|/\epsilon^2)$ is sufficient. More precisely,
with high probability over the one-time draw of $S$, gradient descent on
$\mathcal{R}_n^{S}$ generates, within
$T=\mathcal{O}(1/\epsilon^2)$ iterations, an iterate satisfying
\begin{equation}
\min_{0\leq t\leq T}
\left\|
\nabla\mathcal{R}_n^{G}(\theta_t)
\right\|
\leq
\epsilon.
\end{equation}
Thus, the one-shot method obtains an $\epsilon$-stationarity guarantee for the
fully augmented objective using only
$\mathcal{O}(\log |G|/\epsilon^2)$ sampled transformations.

Note that, although the optimization updates are computed using the fixed partially
augmented objective, the resulting stationarity guarantee is stated with
respect to the fully augmented objective. One-shot augmentation requires only
$\mathcal{O}(\log |G|/\epsilon^2)$ calls to the group-sampling oracle, compared
with the $\mathcal{O}(1/\epsilon^4)$ calls required by streaming stochastic
gradient methods. Table~\ref{tab:nonconvex-comparison} compares the iteration and group-oracle complexities of full, streaming, and one-shot augmentation.

Our results show that the
$\mathcal{O}(1/\epsilon^4)$ group-oracle complexity associated with standard
streaming stochastic-gradient methods is not intrinsic to data augmentation
over finite groups. By exploiting the group structure, one-shot augmentation
reduces this complexity to
$\mathcal{O}(\log |G|/\epsilon^2)$ while retaining a stationarity guarantee for
the fully augmented objective. In particular, continual resampling is not
necessary: a single random subset drawn before optimization can be reused
throughout the entire optimization trajectory.

Our uniform approximation guarantees are obtained by exploiting spectral
properties of group-induced operators together with tools from the
representation theory of finite groups. These techniques control the
dependence between the sampled subset and the resulting optimization trajectory
and may be useful in other optimization problems involving structured random
averages.

\subsection{Contributions}

Our contributions are as follows:
\begin{itemize}
\item
We introduce and analyze \emph{one-shot augmentation} as an alternative to
\emph{streaming augmentation}. One-shot augmentation samples a single subset
of transformations before optimization and reuses it at every iteration,
rather than drawing fresh transformations throughout optimization.

\item
In the smooth nonconvex setting, we prove that gradient descent on the fixed
partially augmented objective generates an $\epsilon$-stationary point of the
fully augmented objective using only
$\mathcal{O}(\log |G|/\epsilon^2)$ calls to the group-sampling oracle. This
improves over the $\mathcal{O}(1/\epsilon^4)$ oracle calls required by standard
streaming stochastic-gradient methods.

\item
We establish uniform approximation guarantees between the partially and fully
augmented objectives and their gradients. The analysis accounts for the
dependence between the fixed random subset and the optimization trajectory
using spectral properties of group-induced operators and tools from
finite-group representation theory.
\end{itemize}

\subsection{Related Work}

Geometric machine learning has emerged as a powerful framework for incorporating structure and symmetries into learning algorithms, with broad applications across scientific domains, including particle physics, molecular modeling, and beyond \citep{bronstein2021geometric,bogatskiy2020lorentz,zhang2025artificial,batzner20223,smidt2021euclidean,batzner2023advancing,weber25}. By leveraging known symmetries, these approaches improve generalization, robustness, and data efficiency, making them particularly well-suited for scientific machine learning tasks.

From a theoretical standpoint, the role of symmetry in learning has been studied through both statistical and computational lenses. On the statistical side, invariances have been shown to yield significant gains in sample complexity and generalization \citep{tahmasebi2023exact}. On the computational side, recent work has begun to characterize the complexity of learning under symmetry constraints~\citep{soleymani2025learning, soleymani2026efficient,kiani2024hardness}. These works highlight that symmetry can fundamentally alter both the statistical and algorithmic properties of learning problems.

Data augmentation is a widely used, model-agnostic technique for incorporating symmetry, where training data are enriched using transformations from a symmetry group. A general group-theoretic framework for data augmentation was developed in \citep{chen2020group}, while kernel-based analyses and feature-based interpretations were studied in \citep{dao2019kernel,shen2022data}. More recent works have investigated the statistical and optimization effects of augmentation, including its role as implicit regularization and its impact on sample efficiency \citep{lin2024good,yang2023sample, tahmasebi2026data}. Despite these advances, the optimization complexity of data augmentation, particularly in regimes with large symmetry groups, remains less understood.

A central challenge in practice is that symmetry groups are often large, making full group averaging computationally infeasible. Recent work has addressed this issue by studying approximate symmetry through sparse averaging \citep{tahmasebi2025achieving}. In particular, it has been shown that approximate symmetry can be enforced using only a logarithmic number of group elements, establishing an exponential gap between exact and approximate symmetry. This result is rooted in tools from representation theory and the spectral properties of random Cayley graphs, building on classical results such as \citep{alon1994random}. Our work builds on this line of research and shows that such sparse approximation guarantees can be directly leveraged to obtain convergence guarantees in optimization with data augmentation.

Finally, our analysis relies on standard results from first-order optimization, including convergence guarantees for gradient descent and stochastic gradient methods in nonconvex settings \citep{Wright_Recht_2022}. We combine these classical tools with recent advances in geometric machine learning to obtain improved guarantees on the number of transformation samples required for optimization under symmetry.

\subsection{Notation}

We use standard asymptotic notation $\mathcal{O}(\cdot)$. When $S$ is a
multiset sampled with replacement, we use $|S|$, by a slight abuse of
notation, for its size counting multiplicities. Thus, if
$S=\{g_1,\ldots,g_m\}$, then $|S|=m$, even when some sampled transformations
coincide. For vectors, we use $\|\cdot\|$ to denote the $\ell_2$ Euclidean
norm.
Let $G$ be a finite set of transformations acting on the data domain $\mathcal{X}$, and for $g \in G$ and $x \in \mathcal{X}$, we write $g \cdot x$ for the action. For any subset $S \subseteq G$, we denote by $\mathcal{R}_n^{S}(\theta)$ the corresponding partially augmented empirical risk, and by $\mathcal{R}_n^{G}(\theta)$ the fully augmented risk.
We say that $\theta$ is an $\epsilon$-stationary point if and only if $\|\nabla \mathcal{R}_n^{G}(\theta)\| \leq \epsilon$.

\section{Problem Statement}

We begin with the empirical risk minimization problem
\begin{equation}
\underset{\theta\in\Theta}{\operatorname{minimize}}
\quad
\mathcal{R}_n(\theta)
\coloneqq
\frac{1}{n}
\sum_{i=1}^n
\ell_i(x_i;\theta),
\label{eq:erm}
\end{equation}
where $\{x_i\}_{i=1}^n$ denote the data,
$\theta\in\Theta\subseteq\mathbb{R}^p$ is the optimization parameter, and
$\ell_i(x_i;\theta)$ is the loss associated with the $i$th sample.

Let $G$ be a finite group acting on the data domain. Our target problem is the
fully augmented empirical risk minimization problem
\begin{equation}
\underset{\theta\in\Theta}{\operatorname{minimize}}
\quad
\mathcal{R}_n^{G}(\theta)
\coloneqq
\frac{1}{n|G|}
\sum_{i=1}^n
\sum_{g\in G}
\ell_i(g\cdot x_i;\theta).
\label{eq:fully-augmented-problem}
\end{equation}
For any nonempty subset $S\subseteq G$, we define the partially augmented
empirical risk
\begin{equation}
\mathcal{R}_n^{S}(\theta)
\coloneqq
\frac{1}{n|S|}
\sum_{i=1}^n
\sum_{g\in S}
\ell_i(g\cdot x_i;\theta).
\label{eq:partially-augmented-risk}
\end{equation}

We consider the following three optimization schemes:
\begin{align}
\text{Full-group GD:}\qquad
\theta_{t+1}
&=
\theta_t-\eta_t\nabla\mathcal{R}_n^{G}(\theta_t),
\label{eq:full-group-gd}
\\
\text{Streaming group-SGD:}\qquad
\theta_{t+1}
&=
\theta_t-\eta_t\nabla\mathcal{R}_n^{S_t}(\theta_t),
\qquad |S_t|=b,
\label{eq:streaming-group-sgd}
\\
\text{One-shot sparse GD:}\qquad
\theta_{t+1}
&=
\theta_t-\eta_t\nabla\mathcal{R}_n^{S}(\theta_t),
\qquad |S|=m,
\label{eq:one-shot-sparse-gd}
\end{align}
where $\eta_t>0$ is the step size. In streaming group-SGD, a fresh random
batch $S_t$ is sampled at each iteration. In one-shot sparse GD, a single
random subset $S$ is sampled before optimization and held fixed throughout
all iterations.

Throughout the paper, we assume that the transformed losses are uniformly
$L$-smooth in the optimization parameter:
\begin{equation}
\left\|
\nabla_\theta\ell_i(g\cdot x_i;\theta)
-
\nabla_\theta\ell_i(g\cdot x_i;\theta')
\right\|
\leq
L\|\theta-\theta'\|
\quad
\text{for all }i\in\{1,\ldots,n\},\ g\in G,\ \text{and }
\theta,\theta'\in\Theta.
\label{eq:smoothness}
\end{equation}
Because averaging preserves the Lipschitz-gradient constant, it follows that
$\mathcal{R}_n^S$ is $L$-smooth for every nonempty multiset $S\subseteq G$;
in particular, $\mathcal{R}_n^G$ is $L$-smooth.

A point $\theta\in\Theta$ is called an $\epsilon$-stationary point of the fully
augmented objective if
\begin{equation}
\left\|
\nabla\mathcal{R}_n^{G}(\theta)
\right\|
\leq
\epsilon.
\label{eq:epsilon-stationarity}
\end{equation}

To quantify access to group transformations, we introduce a
\emph{group-sampling oracle}. Each call to the oracle returns an independent
transformation drawn uniformly from $G$. The \emph{group-oracle complexity} of
an algorithm is the total number of such calls made over the entire
optimization procedure. This notion is distinct from the number of
transformed-loss or transformed-gradient evaluations: once a transformation
has been sampled, it may be reused at multiple parameter iterates without an
additional group-oracle call.

The oracle model captures settings in which obtaining a valid transformation
is itself costly. In particular, a group sample may represent the discovery,
identification, or acquisition of a symmetry of the data, rather than merely
an inexpensive draw from an explicitly enumerated group. Such samples may
therefore constitute valuable information that can be stored and reused. A
broader discussion of this interpretation and its connection to symmetry
discovery is provided in the related-work discussion in the appendix.

Our objective is to determine the group-oracle complexity required by
one-shot sparse GD to produce an $\epsilon$-stationary point of the fully
augmented objective $\mathcal{R}_n^{G}$.

\section{Main Results}
\label{sec:main-results}

We study \emph{one-shot sparse gradient descent}, in which a single collection
of group transformations is sampled before optimization and reused throughout
the entire optimization trajectory. Specifically, let
\begin{equation}
S=\{g_1,\ldots,g_m\},
\qquad
g_1,\ldots,g_m
\overset{\mathrm{i.i.d.}}{\sim}
\operatorname{Unif}(G),
\label{eq:one-shot-sample}
\end{equation}
where $S$ is treated as a multiset. Starting from $\theta_0$, the method
performs
\begin{equation}
\theta_{t+1}
=
\theta_t
-
\eta_t\nabla\mathcal{R}_n^{S}(\theta_t),
\qquad
t=0,\ldots,T-1.
\label{eq:one-shot-update}
\end{equation}
Because the same sampled transformations are reused at every iteration, the
total number of group-oracle calls is exactly $m$, independently of $T$.

\begin{algorithm}[t]
\caption{One-Shot Sparse Gradient Descent}
\label{alg:sparse-aug}
\begin{algorithmic}[1]
\REQUIRE Initial point $\theta_0$, group-sampling oracle, sample size $m$,
step sizes $\{\eta_t\}_{t=0}^{T-1}$, and number of iterations $T$
\STATE Query the group oracle independently $m$ times to obtain
$S=\{g_1,\ldots,g_m\}$
\FOR{$t=0,\ldots,T-1$}
    \STATE
    $\theta_{t+1}
    =
    \theta_t-\eta_t\nabla\mathcal{R}_n^{S}(\theta_t)$
\ENDFOR
\STATE Choose
$\widehat{t}\in
\operatorname*{arg\,min}\limits_{0\leq t<T}
\|\nabla\mathcal{R}_n^{S}(\theta_t)\|$
\RETURN $\widehat{\theta}=\theta_{\widehat{t}}$
\end{algorithmic}
\end{algorithm}

Algorithm~\ref{alg:sparse-aug} optimizes the randomly constructed objective
$\mathcal{R}_n^{S}$, whereas the desired stationarity guarantee concerns the
fully augmented objective $\mathcal{R}_n^{G}$. Moreover, the iterates depend
on the same random sample $S$ used to define the sparse objective. A
pointwise concentration bound at a fixed parameter value is therefore
insufficient. Instead, we establish a single high-probability event on which
\begin{equation}
\sup_{\theta\in\Theta}
\left\|
\nabla\mathcal{R}_n^{S}(\theta)
-
\nabla\mathcal{R}_n^{G}(\theta)
\right\|
\label{eq:uniform-gradient-discrepancy}
\end{equation}
is uniformly controlled.

\subsection{Structural assumptions}

For every sample $i$ and parameter $\theta$, define the vector-valued
gradient function
\begin{equation}
h_{i,\theta}(x)
\coloneqq
\nabla_\theta\ell_i(x;\theta),
\qquad
h_{i,\theta}\colon\mathcal{X}\to\mathbb{R}^p.
\label{eq:gradient-function}
\end{equation}

\begin{assumption}[Invariant gradient RKHS]
\label{ass:gradient-rkhs}
There exists a vector-valued reproducing kernel Hilbert space
$\mathcal{H}$ of functions from $\mathcal{X}$ to $\mathbb{R}^p$ satisfying
the following properties:
\begin{itemize}
    \item For every $i\in\{1,\ldots,n\}$ and $\theta\in\Theta$, we have
    $h_{i,\theta}\in\mathcal{H}$.

    \item The action of $G$ induces a unitary representation
    $U\colon G\to\mathcal{U}(\mathcal{H})$ through
    \begin{equation}
    (U_g h)(x)
    \coloneqq
    h(g^{-1}\cdot x).
    \label{eq:induced-representation}
    \end{equation}

    \item Point evaluation is uniformly bounded: there exists
    $C_{\mathcal H}<\infty$ such that
    \begin{equation}
    \|h\|_{L^\infty}
    \leq
    C_{\mathcal H}\|h\|_{\mathcal H}
    \qquad
    \text{for every }h\in\mathcal H.
    \label{eq:rkhs-linfty-comparison}
    \end{equation}

    \item The gradient functions have uniformly bounded average RKHS norm:
    for some $B_{\mathcal H}<\infty$,
    \begin{equation}
    \sup_{\theta\in\Theta}
    \frac{1}{n}
    \sum_{i=1}^n
    \|h_{i,\theta}\|_{\mathcal H}
    \leq
    B_{\mathcal H}.
    \label{eq:uniform-rkhs-bound}
    \end{equation}
\end{itemize}
\end{assumption}

% \paragraph{Discussion of Assumption~\ref{ass:gradient-rkhs}.}
$\mathcal H$ need not be finite-dimensional; infinite-dimensional RKHSs are
also permitted. The conditions above are satisfied by a broad class of
regular models. In
particular, membership of the gradient functions in a common RKHS arises
naturally for finite-dimensional feature models and kernel-based models, and
the point-evaluation bound is automatic whenever the reproducing kernel is
uniformly bounded. The requirement that the group act unitarily is also
natural for finite groups. Indeed, whenever the function space is preserved
by the group action, an invariant inner product can be obtained by averaging
the original inner product over the group. The principal quantitative
requirement is therefore the uniform bound on the average RKHS norms of the
gradient functions. This condition holds, for example, when the relevant
parameter set is compact and the maps
$\theta\mapsto h_{i,\theta}$ are continuous in the RKHS norm.

The bound in \eqref{eq:rkhs-linfty-comparison} is a standard consequence of
the reproducing property. For the kernel $K$ of $\mathcal H$, each
$K(x,x)$ is a matrix in $\mathbb R^{p\times p}$, and the bound holds whenever
these matrices are uniformly bounded:
\begin{equation}
\sup_{x\in\mathcal X}
\|K(x,x)\|_{\mathrm{op}}^{1/2}
\leq
C_{\mathcal H}.
\label{eq:bounded-kernel}
\end{equation}
This comparison converts an operator-norm estimate in $\mathcal H$ into a
uniform pointwise estimate for the gradient functions.

The unitary representation in
Assumption~\ref{ass:gradient-rkhs} allows us to define the full and sampled
group-averaging operators
\begin{equation}
\Pi_G
\coloneqq
\frac{1}{|G|}
\sum_{g\in G}U_g,
\qquad
\Pi_S
\coloneqq
\frac{1}{m}
\sum_{j=1}^m U_{g_j^{-1}}.
\label{eq:group-averaging-operators}
\end{equation}
The inverse in the definition of $\Pi_S$ aligns the operator with the
augmentation convention: applying $U_{g_j^{-1}}$ evaluates $h$ at
$g_j\cdot x$. Thus, $\Pi_S$ averages over the transformations in $S$.
Because inversion preserves the uniform distribution on $G$, this convention
does not change the spectral concentration bound.
The full averaging operator $\Pi_G$ is the orthogonal projection onto the
$G$-invariant subspace
\begin{equation}
\mathcal{H}^{G}
\coloneqq
\left\{
h\in\mathcal H:
U_g h=h
\text{ for every }g\in G
\right\},
\label{eq:invariant-subspace}
\end{equation}
whereas $\Pi_S$ is its sparse random approximation.

\subsection{Spectral approximation of group averaging}

The sampled operator $\Pi_S$ can be interpreted as a convolution operator
associated with a random Cayley multigraph on $G$. The
representation-theoretic decomposition of this operator separates its
invariant component from its nontrivial irreducible components. Spectral
control of the latter yields the following approximation result. For a
bounded linear operator $A$ on $\mathcal H$, we write
$\|A\|_{\mathrm{op},\mathcal H}
\coloneqq\sup_{h\neq 0}\|Ah\|_{\mathcal H}/\|h\|_{\mathcal H}$.

\begin{proposition}[Spectral approximation of group averaging]
\label{prop:spectral-approximation}
With probability at least $1-\delta$ over the draw of $S$,
\begin{equation}
\left\|
\Pi_S-\Pi_G
\right\|_{\mathrm{op},\mathcal H}
\leq
\min\left\{
1,\,
\sqrt{
\frac{8}{3m}
\log\!\left(\frac{2|G|}{\delta}\right)
}
\right\}.
\label{eq:spectral-operator-bound}
\end{equation}
\end{proposition}

Proposition~\ref{prop:spectral-approximation} is an operator-norm statement
whose high-probability event depends only on $S$ and holds for every unitary
representation of $G$, including infinite-dimensional ones. The
gradient-transfer bound below additionally requires the finite,
space-dependent constants $C_{\mathcal H}$ and $B_{\mathcal H}$ from
Assumption~\ref{ass:gradient-rkhs}. The spectral proof, based on irreducible
representations and random Cayley graphs, appears in the appendix; see also
\citep{tahmasebi2025achieving}.

Combining this spectral estimate with
Assumption~\ref{ass:gradient-rkhs} gives a uniform approximation of the fully
augmented gradient field.

\begin{theorem}[Uniform sparse-to-full gradient approximation]
\label{thm:uniform-gradient-approximation}
Suppose Assumption~\ref{ass:gradient-rkhs} holds. Then, with probability at
least $1-\delta$ over the one-time draw of $S$,
\begin{equation}
\sup_{\theta\in\Theta}
\left\|
\nabla\mathcal{R}_n^{S}(\theta)
-
\nabla\mathcal{R}_n^{G}(\theta)
\right\|
\leq
C_{\mathcal H}B_{\mathcal H}
\sqrt{
\frac{8}{3m}
\log\!\left(\frac{2|G|}{\delta}\right)
}.
\label{eq:uniform-gradient-approximation}
\end{equation}
\end{theorem}

The uniformity in
Theorem~\ref{thm:uniform-gradient-approximation} is essential. It ensures
that the approximation holds simultaneously at every possible iterate of
Algorithm~\ref{alg:sparse-aug}, even though the entire optimization trajectory
depends on $S$.

\subsection{Convergence of one-shot sparse GD}

We next impose the standard regularity conditions required for nonconvex
gradient descent.

\begin{assumption}[Uniform initial optimality gap]
\label{ass:initial-gap}
There exists $\Delta<\infty$ such that, for every sampled multiset $S$,
\begin{equation}
\mathcal{R}_n^{S}(\theta_0)
-
\inf_{\theta\in\Theta}
\mathcal{R}_n^{S}(\theta)
\leq
\Delta.
\label{eq:uniform-initial-gap}
\end{equation}
\end{assumption}

\begin{theorem}[Convergence of one-shot sparse GD]
\label{thm:main}
Suppose Assumptions~\ref{ass:gradient-rkhs}
and~\ref{ass:initial-gap} hold. Run
Algorithm~\ref{alg:sparse-aug} with $\eta_t=1/L$ for every $t$. Then, with probability at
least $1-\delta$,
\begin{equation}
\left\|
\nabla\mathcal{R}_n^{G}(\widehat{\theta})
\right\|
\leq
\sqrt{\frac{2L\Delta}{T}}
+
C_{\mathcal H}B_{\mathcal H}
\sqrt{
\frac{8}{3m}
\log\!\left(\frac{2|G|}{\delta}\right)
}.
\label{eq:main-convergence-bound}
\end{equation}
Consequently, choosing
\begin{equation}
T
\geq
\frac{8L\Delta}{\epsilon^2}
\qquad\text{and}\qquad
m
\geq
\frac{
32C_{\mathcal H}^2B_{\mathcal H}^2
}{
3\epsilon^2
}
\log\!\left(\frac{2|G|}{\delta}\right)
\label{eq:main-parameter-choice}
\end{equation}
ensures that
\begin{equation}
\left\|
\nabla\mathcal{R}_n^{G}(\widehat{\theta})
\right\|
\leq
\epsilon
\end{equation}
with probability at least $1-\delta$.
\end{theorem}

Theorem~\ref{thm:main} gives the group-oracle complexity
\begin{equation}
m
=
\mathcal{O}\left(
\frac{
C_{\mathcal H}^2B_{\mathcal H}^2
\bigl(\log |G|+\log(1/\delta)\bigr)
}{
\epsilon^2
}
\right).
\label{eq:main-oracle-complexity}
\end{equation}
In particular, when the regularity constants are suppressed, one-shot sparse
GD requires
\begin{equation}
\mathcal{O}\left(
\frac{\log |G|+\log(1/\delta)}{\epsilon^2}
\right)
\end{equation}
group-oracle calls and
$\mathcal{O}(1/\epsilon^2)$ optimization iterations. All group transformations
are sampled before optimization and subsequently reused, so no additional
group-oracle calls are required during training.

\section{Proof Sketch}
\label{sec:proof-overview}

This section outlines the main ideas behind
Theorem~\ref{thm:main}. Complete proofs, including the spectral approximation
result and all auxiliary lemmas, are deferred to the appendix.

Let $D_S\coloneqq\Pi_S-\Pi_G$ denote the difference between the sampled and
full group-averaging operators. Since $\Pi_G$ is the orthogonal projection
onto the invariant subspace, $D_S$ is supported on the nontrivial
representation components.

To bound the empirical deviation operator $D_S$, we use the fact that,
although $\mathcal H$ may be infinite-dimensional, the representation of the
finite group $G$ decomposes into finite-dimensional irreducible
representations. Each irreducible representation may occur with an arbitrary
multiplicity, but the corresponding matrix block is merely repeated, so its
multiplicity does not affect the operator norm. We may therefore apply
finite-dimensional matrix concentration to the finitely many distinct
irreducible blocks and then take a union bound. The resulting event depends
only on $S$ and controls every unitary representation of $G$ simultaneously,
even one selected after observing $S$. This gives, with probability at least
$1-\delta$,
\begin{equation}
\|D_S\|_{\mathrm{op},\mathcal H}
\leq
\sqrt{
\frac{8}{3m}
\log\!\left(\frac{2|G|}{\delta}\right)
}.
\label{eq:proof-overview-spectral}
\end{equation}
Crucially, \eqref{eq:proof-overview-spectral} controls an operator norm and hence holds
simultaneously for every gradient function in $\mathcal H$.
Using the group-averaging operators, the difference between the sparse and
fully augmented gradients can be expressed as
\begin{equation}
\nabla\mathcal R_n^S(\theta)
-
\nabla\mathcal R_n^G(\theta)
=
\frac{1}{n}
\sum_{i=1}^n
\bigl(D_S h_{i,\theta}\bigr)(x_i),
\label{eq:proof-overview-gradient-identity}
\end{equation}

The point-evaluation and RKHS-norm bounds in
Assumption~\ref{ass:gradient-rkhs}, together with
\eqref{eq:proof-overview-spectral}, give
\begin{equation}
\sup_{\theta\in\Theta}
\left\|
\nabla\mathcal R_n^S(\theta)
-
\nabla\mathcal R_n^G(\theta)
\right\|
\leq
C_{\mathcal H}B_{\mathcal H}
\sqrt{
\frac{8}{3m}
\log\!\left(\frac{2|G|}{\delta}\right)
}.
\label{eq:proof-overview-uniform}
\end{equation}

This uniformity is essential: every iterate depends on $S$, so a pointwise
concentration bound at a parameter chosen independently of $S$ would not
suffice.

By \eqref{eq:smoothness}, the sparse objective is $L$-smooth. The standard
descent estimate for gradient descent with $\eta_t=1/L$, together with
Assumption~\ref{ass:initial-gap}, gives
\begin{equation}
\min_{0\leq t<T}
\left\|
\nabla\mathcal R_n^S(\theta_t)
\right\|
\leq
\sqrt{\frac{2L\Delta}{T}}.
\label{eq:proof-overview-sparse-stationarity}
\end{equation}
Since Algorithm~\ref{alg:sparse-aug} returns the iterate with the smallest
sparse-gradient norm, \eqref{eq:proof-overview-uniform} and the triangle
inequality yield
\begin{align}
\left\|
\nabla\mathcal R_n^G(\widehat{\theta})
\right\|
&\leq
\left\|
\nabla\mathcal R_n^S(\widehat{\theta})
\right\|
+
\left\|
\nabla\mathcal R_n^G(\widehat{\theta})
-
\nabla\mathcal R_n^S(\widehat{\theta})
\right\|
\nonumber\\
&\leq
\sqrt{\frac{2L\Delta}{T}}
+
C_{\mathcal H}B_{\mathcal H}
\sqrt{
\frac{8}{3m}
\log\!\left(\frac{2|G|}{\delta}\right)
}.
\label{eq:proof-overview-final}
\end{align}
The first term is the usual finite-iteration optimization error; the second
is the error from replacing the full group average by its fixed sparse
approximation. The choices of $T$ and $m$ in
Theorem~\ref{thm:main} make both terms at most $\epsilon/2$, completing the
argument. The appendix supplies the full spectral and optimization proofs.

\section{Experiments}
\label{sec:experiment}

In this section, we provide a small-scale experiment to validate our theory. We consider a permutation-invariant regression
experiment for which the fully augmented objective can be evaluated exactly
throughout training. The predictor is constructed from a Gaussian kernel,
while a factorized parameterization makes the optimization problem nonconvex.

\subsection{Permutation-invariant sum regression}

Let $G=S_6$ act on $\mathbb{R}^6$ by coordinate permutations, so that
$|G|=6!=720$. Each input has independent coordinates sampled uniformly from
$[-1,1]$ and is then sorted in increasing order. The regression target is the
plain coordinate sum
\begin{equation}
y(x)=\sum_{j=1}^{6}x_j.
\end{equation}
This target is exactly permutation invariant:
$y(g\cdot x)=y(x)$ for every $g\in G$. Sorting provides a canonical
representation during training, however, and therefore allows an unaugmented
predictor to fit the observed ordering without learning the desired behavior
over the complete permutation orbit. We independently generate 128 training
samples and 256 test samples.

We use the Gaussian kernel
\begin{equation}
k_\sigma(x,z)
=
\exp\left(-\frac{\|x-z\|_2^2}{2\sigma^2}\right),
\qquad \sigma=1.25.
\end{equation}
Coordinate permutations preserve Euclidean distance and hence act unitarily on
the corresponding Gaussian RKHS. We draw 160 fixed kernel centers
$\{z_r\}_{r=1}^{160}$ independently and uniformly from $[-1,1]^6$, without
sorting them, and use the factorized model
\begin{equation}
f_{a,b}(x)
=
\sum_{r=1}^{160}a_rb_r k_\sigma(x,z_r).
\end{equation}
The loss is
$\ell(x,y;a,b)=\frac12(f_{a,b}(x)-y)^2$. We optimize $a$ and $b$ jointly and
project both vectors onto $[-3,3]^{160}$ after every update. The bilinear
coefficients $a_rb_r$ make the parameterized empirical risk nonconvex. The
Gaussian-kernel construction and bounded parameter domain also provide the
controlled function space used in our theory.

\subsection{Methods and evaluation protocol}

We compare no augmentation; full-group GD, which averages over all 720
permutations; streaming group-SGD with one fresh permutation per iteration
($b=1$); and one-shot sparse GD with the three predeclared subset sizes
\begin{equation}
|S|\in\{4,16,64\}.
\end{equation}
Streaming and one-shot transformations are drawn uniformly with replacement.
All methods use the same initialization, full data batches, a common step size
of $0.05$, and 500 iterations. We repeat the comparison over 10 independent
seeds. Within each seed, all
methods share the data, kernel centers, and initialization; these quantities
and the sampled transformations are generated independently across seeds.

Every 25 iterations, we enumerate all of $S_6$ to compute the fully augmented
training risk and full-gradient norm
$\|\nabla\mathcal R_n^G(\theta_t)\|$. We define the permutation-averaged test
risk as
\begin{equation}
\mathcal R_{\mathrm{test}}^G(a,b)
\coloneqq
\frac{1}{n_{\mathrm{test}}|G|}
\sum_{i=1}^{n_{\mathrm{test}}}
\sum_{g\in G}
\ell(g\cdot x_i^{\mathrm{test}},y_i^{\mathrm{test}};a,b).
\end{equation}
Thus, the test loss is averaged over both test samples and all $720$
permutations of each sample.
The identity-order test risk is evaluated separately to diagnose reliance on
the sorted-input representation. Group-oracle calls count fresh permutation
samples: $|S|$ for one-shot augmentation, 720 for full-group GD, and one per
iteration for streaming group-SGD.

\begin{figure}[t]
  \centering
  \IfFileExists{kernel_permutation_trajectories.pdf}{%
    \includegraphics[width=\linewidth]{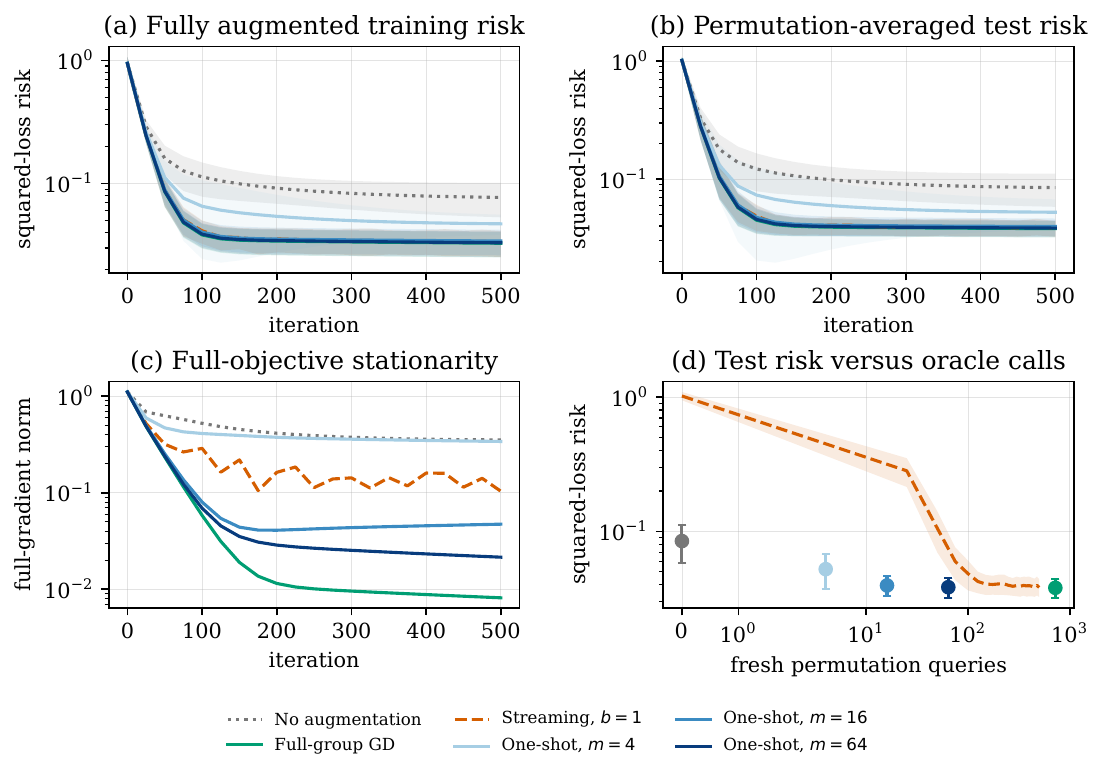}%
  }{\fbox{\parbox{0.9\linewidth}{\centering Gaussian-kernel trajectories unavailable.}}}
  \caption{Sum regression with a Gaussian-kernel predictor under $S_6$.
  \textbf{(a)} Fully augmented training risk. \textbf{(b)}
  Permutation-averaged test risk. \textbf{(c)} Gradient norm of the fully
  augmented objective. \textbf{(d)} Test risk versus fresh permutation
  queries. Curves show means over 10 seeds. Shaded regions and error bars in
  the risk plots represent one standard deviation.}
  \label{fig:kernel-permutation}
\end{figure}

\subsection{Results}

\hyperref[fig:kernel-permutation]{Figure~\ref*{fig:kernel-permutation},
Panels~\textbf{(a)}--\textbf{(b)}} shows that the augmented methods rapidly
reduce both training and test risk relative to the unaugmented predictor.
Increasing the fixed subset from 4 to 16 substantially improves the one-shot
trajectory, while the improvement from 16 to 64 is smaller. With 64 fixed
permutations, one-shot sparse GD closely tracks the test risk of both
full-group GD and streaming group-SGD while using substantially fewer fresh
group samples, as shown in
\hyperref[fig:kernel-permutation]{Figure~\ref*{fig:kernel-permutation},
Panel~\textbf{(d)}}.

\hyperref[fig:kernel-permutation]{Figure~\ref*{fig:kernel-permutation},
Panel~\textbf{(c)}} shows the corresponding behavior for stationarity of the fully
augmented objective. Full-group GD attains the smallest gradient norm, while
the one-shot trajectory moves progressively closer to it as $|S|$ increases.
Streaming group-SGD reaches low test risk, but its full-gradient norm continues
to fluctuate because it uses one fresh transformation per update. Together,
these results illustrate the predicted separation between reusable group
samples and optimization iterations for a group of size 720.

\section{Conclusion}

We studied the optimization of fully augmented empirical-risk objectives
defined by finite transformation groups. We showed that transformations need
not be resampled throughout optimization: a single sparse collection sampled
at initialization can be reused by gradient descent while still providing a
stationarity guarantee for the fully augmented objective. Under our
structural and smoothness assumptions, one-shot sparse gradient descent
requires
$\mathcal{O}\bigl((\log |G|+\log(1/\delta))/\epsilon^2\bigr)$
group-oracle calls to obtain an $\epsilon$-stationary point with probability
at least $1-\delta$, improving on the
$\mathcal{O}(1/\epsilon^4)$ transformation-query complexity of standard
group-SGD. Our analysis combines spectral approximation of group-averaging
operators with uniform control of the augmented gradient field. These results
suggest that previously discovered transformations can be treated as
reusable optimization resources, rather than repeatedly sampled throughout
training.

\section*{Acknowledgements}

BT and MW were partially supported by NSF Awards CBET-2112085 and DMS-2406905, by the Army Research Office (W911NF-26-1-A060), and the Air Force Office of Scientific Research (FA9550261B044). MW acknowledges partial funding from an Alfred P. Sloan Fellowship in Mathematics and the AI2050 program at Schmidt Sciences (G-25-69786).
The U.S. Government is authorized to reproduce and distribute reprints for Governmental purposes notwithstanding any copyright notation thereon. The opinions, findings, views, conclusions or recommendations contained herein are those of the authors and should not be interpreted as necessarily representing the official policies or endorsements, either expressed or implied, of the DAF, AFRL, ARO or the U.S. Government. 

\section*{LLM usage disclosure}
LLM tools were used for copyediting, including
improvements to grammar, wording, clarity, and \LaTeX{} presentation. The
research direction, core ideas, theoretical results, and experimental design
are original contributions of the authors. The authors carefully reviewed and
verified all technical content and take full responsibility for the final
manuscript.

%%%%%%%%%%%%%%%%%%%%%%%%%%%%%%%%%%%%%%%%%%%%%%%%%%%%%%%%%%%%%%%%%%%%%%%%%%%%%%%% 

 % \bibliographystyle{ieeetr}
  \bibliographystyle{plainnat}
 \bibliography{ref}

@string{aistats = {Int. Conference on Artificial Intelligence and Statistics (AISTATS)}}

@string{icml = {Int. Conference on Machine Learning (ICML)}}

@string{iclr = {Int. Conference on Learning Representations (ICLR)}}

@string{colt = {Conference on Learning Theory (COLT)}}

@string{neurips = {Advances in Neural Information Processing Systems (NeurIPS)}}

@inproceedings{tahmasebi2025achieving,
  title={Achieving Approximate Symmetry Is Exponentially Easier than Exact Symmetry},
  author={Tahmasebi, Behrooz and Weber, Melanie},
  booktitle=iclr,
  year=2026
}

@article{bronstein2021geometric,
  title={Geometric deep learning: Grids, groups, graphs, geodesics, and gauges},
  author={Bronstein, Michael M and Bruna, Joan and Cohen, Taco and Veli{\v{c}}kovi{\'c}, Petar},
  journal={arXiv preprint arXiv:2104.13478},
  year={2021}
}

@inproceedings{bogatskiy2020lorentz,
  title={Lorentz group equivariant neural network for particle physics},
  author={Bogatskiy, Alexander and Anderson, Brandon and Offermann, Jan and Roussi, Marwah and Miller, David and Kondor, Risi},
  booktitle=icml,
  pages={992--1002},
  year={2020},
  organization={PMLR}
}

@article{zhang2025artificial,
  title={Artificial intelligence for science in quantum, atomistic, and continuum systems},
  author={Zhang, Xuan and Wang, Limei and Helwig, Jacob and Luo, Youzhi and Fu, Cong and Xie, Yaochen and Liu, Meng and Lin, Yuchao and Xu, Zhao and Yan, Keqiang and others},
  journal={Foundations and Trends{\textregistered} in Machine Learning},
  volume={18},
  number={4},
  pages={385--849},
  year={2025},
  publisher={Emerald Publishing Limited}
}

@article{alon1994random,
  title={Random Cayley graphs and expanders},
  author={Alon, Noga and Roichman, Yuval},
  journal={Random Structures \& Algorithms},
  volume={5},
  number={2},
  pages={271--284},
  year={1994},
  publisher={Wiley Online Library}
}

@article{batzner20223,
  title={E(3)-equivariant graph neural networks for data-efficient and accurate interatomic potentials},
  author={Batzner, Simon and Musaelian, Albert and Sun, Lixin and Geiger, Mario and Mailoa, Jonathan P and Kornbluth, Mordechai and Molinari, Nicola and Smidt, Tess E and Kozinsky, Boris},
  journal={Nature communications},
  volume={13},
  number={1},
  pages={2453},
  year={2022},
  publisher={Nature Publishing Group UK London}
}

@article{smidt2021euclidean,
  title={Euclidean symmetry and equivariance in machine learning},
  author={Smidt, Tess E},
  journal={Trends in Chemistry},
  volume={3},
  number={2},
  pages={82--85},
  year={2021},
  publisher={Elsevier}
}

@article{batzner2023advancing,
  title={Advancing molecular simulation with equivariant interatomic potentials},
  author={Batzner, Simon and Musaelian, Albert and Kozinsky, Boris},
  journal={Nature Reviews Physics},
  volume={5},
  number={8},
  pages={437--438},
  year={2023},
  publisher={Nature Publishing Group UK London}
}

@inproceedings{tahmasebi2023exact,
  title={The exact sample complexity gain from invariances for kernel regression},
  author={Tahmasebi, Behrooz and Jegelka, Stefanie},
  booktitle=neurips,
  volume={36},
  pages={55616--55646},
  year={2023}
}

@inproceedings{
soleymani2025learning,
title={Learning with Exact Invariances in Polynomial Time},
author={Ashkan Soleymani and Behrooz Tahmasebi and Stefanie Jegelka and Patrick Jaillet},
booktitle=icml,
year={2025},
}

@inproceedings{shen2022data,
  title={Data augmentation as feature manipulation},
  author={Shen, Ruoqi and Bubeck, S{\'e}bastien and Gunasekar, Suriya},
  booktitle=icml,
  pages={19773--19808},
  year={2022},
  organization={PMLR}
}

@article{lin2024good,
  title={The good, the bad and the ugly sides of data augmentation: An implicit spectral regularization perspective},
  author={Lin, Chi-Heng and Kaushik, Chiraag and Dyer, Eva L and Muthukumar, Vidya},
  journal={Journal of Machine Learning Research},
  volume={25},
  number={91},
  pages={1--85},
  year={2024}
}

@inproceedings{yang2023sample,
  title={Sample efficiency of data augmentation consistency regularization},
  author={Yang, Shuo and Dong, Yijun and Ward, Rachel and Dhillon, Inderjit S and Sanghavi, Sujay and Lei, Qi},
  booktitle={International Conference on Artificial Intelligence and Statistics},
  pages={3825--3853},
  year={2023},
  organization={PMLR}
}

@article{chen2020group,
  title={A group-theoretic framework for data augmentation},
  author={Chen, Shuxiao and Dobriban, Edgar and Lee, Jane H},
  journal={Journal of Machine Learning Research},
  volume={21},
  number={245},
  pages={1--71},
  year={2020}
}

@inproceedings{dao2019kernel,
  title={A kernel theory of modern data augmentation},
  author={Dao, Tri and Gu, Albert and Ratner, Alexander and Smith, Virginia and De Sa, Chris and R{\'e}, Christopher},
  booktitle=icml,
  pages={1528--1537},
  year={2019},
  organization={PMLR}
}

@inproceedings{soleymani2026efficient,
  title={Efficient Learning and Symmetry Discovery under Exact Invariances},
  author={Soleymani, Ashkan and Tahmasebi, Behrooz and Jaillet, Patrick and Jegelka, Stefanie},
  booktitle=colt,
  year={2026}
}

@inproceedings{tahmasebi2026data,
  title={Data augmentation: A {Fourier} analysis perspective},
  author={Tahmasebi, Behrooz and Weber, Melanie and Jegelka, Stefanie},
  booktitle=colt,
  year={2026}
}

@book{Wright_Recht_2022,
 place={Cambridge}, 
 title={Optimization for Data Analysis}, 
 publisher={Cambridge University Press}, 
 author={Wright, Stephen J. and Recht, Benjamin}, 
 year={2022}
 }

@inproceedings{lin2024equivariance,
  title={Equivariance via Minimal Frame Averaging for More Symmetries and Efficiency},
  author={Lin, Yuchao and Helwig, Jacob and Gui, Shurui and Ji, Shuiwang},
  booktitle=icml,
  year={2024}
}

@inproceedings{ma2024canonicalization,
  title={A canonicalization perspective on invariant and equivariant learning},
  author={Ma, George and Wang, Yifei and Lim, Derek and Jegelka, Stefanie and Wang, Yisen},
  booktitle=neurips,
  year={2024}
}

@inproceedings{tahmasebi2025regularity,
  title={Regularity in Canonicalized Models: A Theoretical Perspective},
  author={Tahmasebi, Behrooz and Jegelka, Stefanie},
  booktitle=aistats,
  year={2025}
}

@inproceedings{tahmasebigeneralization,
  title={Generalization Bounds for Canonicalization: A Comparative Study with Group Averaging},
  author={Tahmasebi, Behrooz and Jegelka, Stefanie},
  booktitle=iclr,
  year={2025}
}

@inproceedings{dym2024equivariant,
  title={Equivariant frames and the impossibility of continuous canonicalization},
  author={Dym, Nadav and Lawrence, Hannah and Siegel, Jonathan W},
  booktitle=icml,
  year={2024}
}

@inproceedings{shumaylovlie,
  title={Lie Algebra Canonicalization: Equivariant Neural Operators under arbitrary Lie Groups},
  author={Shumaylov, Zakhar and Zaika, Peter and Rowbottom, James and Sherry, Ferdia and Weber, Melanie and Sch{\"o}nlieb, Carola-Bibiane},
  booktitle=iclr,
  year={2025}
}

@inproceedings{punyframe,
  title={Frame Averaging for Invariant and Equivariant Network Design},
  author={Puny, Omri and Atzmon, Matan and Smith, Edward J and Misra, Ishan and Grover, Aditya and Ben-Hamu, Heli and Lipman, Yaron},
  booktitle=iclr,
  year={2022}
}

@inproceedings{kaba2023equivariance,
  title={Equivariance with learned canonicalization functions},
  author={Kaba, S{\'e}kou-Oumar and Mondal, Arnab Kumar and Zhang, Yan and Bengio, Yoshua and Ravanbakhsh, Siamak},
  booktitle=icml,
  year={2023}
}

@inproceedings{soleymani2025robust,
  title={A Robust Kernel Statistical Test of Invariance: Detecting Subtle Asymmetries},
  author={Soleymani, Ashkan and Tahmasebi, Behrooz and Jegelka, Stefanie and Jaillet, Patrick},
  booktitle=aistats,
  year={2025}
}

@inproceedings{
soleymani2026a,
title={A Unified Framework for Statistical Testing of Invariance},
author={Ashkan Soleymani and Behrooz Tahmasebi and Patrick Jaillet and Stefanie Jegelka},
booktitle={The ICML 2026 Workshop on Hypothesis Testing},
year={2026}
}

@inproceedings{tahmasebi2023sample,
  title={Sample complexity bounds for estimating probability divergences under invariances},
  author={Tahmasebi, Behrooz and Jegelka, Stefanie},
  booktitle=icml,
  year={2024}
}

@inproceedings{chen2023sample,
  title={Sample complexity of probability divergences under group symmetry},
  author={Chen, Ziyu and Katsoulakis, Markos and Rey-Bellet, Luc and Zhu, Wei},
  booktitle=icml,
  year={2023}
}

@inproceedings{
tahmasebi2026adaptive,
title={Adaptive Symmetry Discovery for Dynamical System Identification},
author={Behrooz Tahmasebi and Melanie Weber},
booktitle=icml,
year={2026}
}

@article{tropp2012user,
  title={User-friendly tail bounds for sums of random matrices},
  author={Tropp, Joel A},
  journal={Foundations of computational mathematics},
  volume={12},
  number={4},
  pages={389--434},
  year={2012},
  publisher={Springer}
}

@inproceedings{hanin2021data,
  title={How data augmentation affects optimization for linear regression},
  author={Hanin, Boris and Sun, Yi},
  booktitle=neurips,
  year={2021}
}

@inproceedings{shao2022theory,
  title={A theory of pac learnability under transformation invariances},
  author={Shao, Han and Montasser, Omar and Blum, Avrim},
  booktitle=neurips,
  year={2022}
}

@InProceedings{pmlr-v134-mei21a,
  title = 	 {Learning with invariances in random features and kernel models},
  author =       {Mei, Song and Misiakiewicz, Theodor and Montanari, Andrea},
  booktitle = 	 colt,
  year = 	 {2021}
}

@InProceedings{pmlr-v139-elesedy21a,
  title = 	 {Provably Strict Generalisation Benefit for Equivariant Models},
  author =       {Elesedy, Bryn and Zaidi, Sheheryar},
  booktitle = 	 icml,
  year = 	 {2021}
}

@InProceedings{zhu2021understanding,
  title={Understanding the generalization benefit of model invariance from a data perspective},
  author={Zhu, Sicheng and An, Bang and Huang, Furong},
  booktitle=neurips,
  year={2021}
}

@InProceedings{pmlr-v202-hounie23a,
  title = 	 {Automatic Data Augmentation via Invariance-Constrained Learning},
  author =       {Hounie, Ignacio and Chamon, Luiz F. O. and Ribeiro, Alejandro},
  booktitle = 	icml,
  year = 	 {2023}
}

@InProceedings{benton2020learning,
  title={Learning invariances in neural networks from training data},
  author={Benton, Gregory and Finzi, Marc and Izmailov, Pavel and Wilson, Andrew G},
  booktitle=neurips,
  year={2020}
}

@InProceedings{immer2022invariance,
  title={Invariance learning in deep neural networks with differentiable laplace approximations},
  author={Immer, Alexander and van der Ouderaa, Tycho and R{\"a}tsch, Gunnar and Fortuin, Vincent and van der Wilk, Mark},
  booktitle=neurips,
  year={2022}
}

@InProceedings{kvinge2022ways,
  title={In what ways are deep neural networks invariant and how should we measure this?},
  author={Kvinge, Henry and Emerson, Tegan and Jorgenson, Grayson and Vasquez, Scott and Doster, Tim and Lew, Jesse},
  booktitle=neurips,
  year={2022}
}

@InProceedings{bouchacourt2021grounding,
  title={Grounding inductive biases in natural images: invariance stems from variations in data},
  author={Bouchacourt, Diane and Ibrahim, Mark and Morcos, Ari},
  booktitle=neurips,
  year={2021}
}

@InProceedings{pmlr-v162-bachmann22a,
  title = 	 {How Tempering Fixes Data Augmentation in {B}ayesian Neural Networks},
  author =       {Bachmann, Gregor and Noci, Lorenzo and Hofmann, Thomas},
  booktitle = 	 icml,
  year = 	 {2022}
}

@inproceedings{
finzi2021residual,
title={Residual Pathway Priors for Soft Equivariance Constraints},
author={Marc Anton Finzi and Gregory Benton and Andrew Gordon Wilson},
booktitle=neurips,
year={2021}
}

@inproceedings{
ouderaa2022relaxing,
title={Relaxing Equivariance Constraints with Non-stationary Continuous Filters},
author={Tycho F.A. van der Ouderaa and David W. Romero and Mark van der Wilk},
booktitle=neurips,
year={2022}
}

@inproceedings{ashman2024approximately,
  title={Approximately equivariant neural processes},
  author={Ashman, Matthew and Diaconu, Cristiana and Weller, Adrian and Bruinsma, Wessel and Turner, Richard E},
  booktitle=neurips,
  year={2024}
}

@inproceedings{huang2023approximately,
  title={Approximately equivariant graph networks},
  author={Huang, Ningyuan and Levie, Ron and Villar, Soledad},
  booktitle=neurips,
  year={2023}
}

@inproceedings{
ouderaa2023learning,
title={Learning Layer-wise Equivariances Automatically using Gradients},
author={Tycho F.A. van der Ouderaa and Alexander Immer and Mark van der Wilk},
booktitle=neurips,
year={2023}
}

@InProceedings{pmlr-v151-yeh22b,
  title = 	 { Equivariance Discovery by Learned Parameter-Sharing },
  author =       {Yeh, Raymond A. and Hu, Yuan-Ting and Hasegawa-Johnson, Mark and Schwing, Alexander},
  booktitle = 	 aistats,
  year = 	 {2022}
}

@InProceedings{dehmamy2021automatic,
  title={Automatic symmetry discovery with lie algebra convolutional network},
  author={Dehmamy, Nima and Walters, Robin and Liu, Yanchen and Wang, Dashun and Yu, Rose},
  booktitle=neurips,
  year={2021}
}

@InProceedings{pmlr-v202-yang23n,
  title = 	 {Generative Adversarial Symmetry Discovery},
  author =       {Yang, Jianke and Walters, Robin and Dehmamy, Nima and Yu, Rose},
  booktitle = 	icml,
  year = 	 {2023}
}

@inproceedings{
huh2025a,
title={Discovering Group Structures via Unitary Representation Learning},
author={Dongsung Huh},
booktitle=iclr,
year={2025}
}

@InProceedings{pmlr-v162-park22a,
  title = 	 {Learning Symmetric Embeddings for Equivariant World Models},
  author =       {Park, Jung Yeon and Biza, Ondrej and Zhao, Linfeng and Van De Meent, Jan-Willem and Walters, Robin},
  booktitle = 	 icml,
  year = 	 {2022}
}

@InProceedings{pmlr-v238-karjol24a,
  title = 	 {A Unified Framework for Discovering Discrete Symmetries},
  author =       {Karjol, Pavan and Kashyap, Rohan and Gopalan, Aditya and Prathosh, A. P.},
  booktitle = 	aistats,
  year = 	 {2024}
}

@article{ghadimi2013stochastic,
  title={Stochastic first-and zeroth-order methods for nonconvex stochastic programming},
  author={Ghadimi, Saeed and Lan, Guanghui},
  journal={SIAM journal on optimization},
  volume={23},
  number={4},
  pages={2341--2368},
  year={2013},
  publisher={SIAM}
}

@InProceedings{pmlr-v48-reddi16,
  title = 	 {Stochastic Variance Reduction for Nonconvex Optimization},
  author = 	 {Reddi, Sashank J. and Hefny, Ahmed and Sra, Suvrit and Poczos, Barnabas and Smola, Alex},
  booktitle = 	 icml,
  year = 	 {2016}
}

@InProceedings{pmlr-v70-nguyen17b,
  title = 	 {{SARAH}: A Novel Method for Machine Learning Problems Using Stochastic Recursive Gradient},
  author =       {Lam M. Nguyen and Jie Liu and Katya Scheinberg and Martin Tak{\'a}{\v{c}}},
  booktitle = 	icml,
  year = 	 {2017}
}

@InProceedings{fang2018spider,
  title={Spider: Near-optimal non-convex optimization via stochastic path-integrated differential estimator},
  author={Fang, Cong and Li, Chris Junchi and Lin, Zhouchen and Zhang, Tong},
  booktitle=neurips,
  year={2018}
}

@InProceedings{cutkosky2019momentum,
  title={Momentum-based variance reduction in non-convex sgd},
  author={Cutkosky, Ashok and Orabona, Francesco},
  booktitle=neurips,
  year={2019}
}

@InProceedings{pmlr-v139-li21a,
  title = 	 {PAGE: A Simple and Optimal Probabilistic Gradient Estimator for Nonconvex Optimization},
  author =       {Li, Zhize and Bao, Hongyan and Zhang, Xiangliang and Richtarik, Peter},
  booktitle = 	 icml,
  year = 	 {2021}
}

@InProceedings{ahn2020sgd,
  title={{SGD} with shuffling: optimal rates without component convexity and large epoch requirements},
  author={Ahn, Kwangjun and Yun, Chulhee and Sra, Suvrit},
  booktitle=neurips,
  year={2020}
}

@InProceedings{mishchenko2020random,
  title={Random reshuffling: Simple analysis with vast improvements},
  author={Mishchenko, Konstantin and Khaled, Ahmed and Richt{\'a}rik, Peter},
  booktitle=neurips,
  year={2020}
}

@InProceedings{pmlr-v97-zhou19b,
  title = 	 {Lower Bounds for Smooth Nonconvex Finite-Sum Optimization},
  author =       {Zhou, Dongruo and Gu, Quanquan},
  booktitle = 	 icml,
  year = 	 {2019}
}

@article{weber25,
  title={Geometric machine learning},
  author={Weber, Melanie},
  journal={AI Magazine},
  year={2025}
}

@InProceedings{pmlr-v291-saad25a,
  title = 	 {New Lower Bounds for Non-Convex Stochastic Optimization through Divergence Decomposition},
  author =       {Saad, El Mehdi and Lee, Wei-Cheng and Orabona, Francesco},
  booktitle = 	colt,
  year = 	 {2025}
}

@inproceedings{kiani2024hardness,
  title={On the hardness of learning under symmetries},
  author={Kiani, Bobak T and Le, Thien and Lawrence, Hannah and Jegelka, Stefanie and Weber, Melanie},
  booktitle={International Conference on Learning Representations},
  year={2024}
}

 \clearpage
 \appendix

\section{Additional Related Work}
\label{app:additional-related-work}

\paragraph{Data augmentation and invariant learning.}
The statistical and optimization effects of data augmentation have been
studied in several settings. For linear regression, augmented gradient methods
can be viewed as optimization over time-varying objectives
\citep{hanin2021data}. Invariance also affects PAC learnability
\citep{shao2022theory}, the efficiency of kernel and random-feature methods
\citep{pmlr-v134-mei21a}, generalization under equivariant averaging
\citep{pmlr-v139-elesedy21a}, and transformation-induced sample covers
\citep{zhu2021understanding}. Other works learn or adapt the augmentation
distribution \citep{pmlr-v202-hounie23a,benton2020learning,immer2022invariance},
measure the invariance acquired by a model \citep{kvinge2022ways}, or study
how augmentation interacts with the training distribution and Bayesian
uncertainty \citep{bouchacourt2021grounding,pmlr-v162-bachmann22a}. Our focus
is different: we study the number of group-oracle calls needed to optimize the
fully augmented objective when sampled transformations can be reused.

\paragraph{Alternative mechanisms for enforcing symmetry.}
Frame averaging enforces equivariance by averaging over input-dependent frames
\citep{punyframe}; minimal frame averaging reduces the number of frames needed
\citep{lin2024equivariance}. Canonicalization instead selects an orbit
representative, either directly or through a learned map
\citep{kaba2023equivariance,ma2024canonicalization}, and has been extended to
Lie group actions and neural operators \citep{shumaylovlie}. Continuous
canonicalization may not exist for some actions \citep{dym2024equivariant},
and its statistical behavior can differ from group averaging
\citep{tahmasebigeneralization}. Nevertheless, the end-to-end model may remain
regular even when the canonicalization map is discontinuous
\citep{tahmasebi2025regularity}. We do not modify the architecture or select
orbit representatives. We retain the augmented objective and approximate its
full group average using a fixed sparse sample.

\paragraph{Approximate and data-adaptive equivariance.}
When symmetry is only approximate, hard architectural constraints may be too
restrictive. This motivates residual equivariant pathways
\citep{finzi2021residual}, non-stationary filters
\citep{ouderaa2022relaxing}, approximately equivariant neural processes and
graph networks \citep{ashman2024approximately,huang2023approximately}, and
methods that learn layerwise constraints or parameter-sharing patterns
\citep{ouderaa2023learning,pmlr-v151-yeh22b}. These approaches relax the model
class. In our setting, the model and full objective are unchanged; only the
group average used during optimization is sparsified.

\paragraph{Symmetry testing and discovery.}
Statistical tests can determine whether data support a proposed invariance
\citep{soleymani2025robust,soleymani2026a}, while invariance can improve the
estimation of probability divergences
\citep{chen2023sample,tahmasebi2023sample}. Related methods discover
continuous generators \citep{dehmamy2021automatic}, distribution-preserving
transformations \citep{pmlr-v202-yang23n}, unitary representations
\citep{huh2025a}, symmetric embeddings \citep{pmlr-v162-park22a}, or discrete
symmetry groups \citep{pmlr-v238-karjol24a}. Symmetry discovery can also be
coupled with system identification \citep{tahmasebi2026adaptive}. These works
support our oracle interpretation: acquiring a valid transformation may be
costly, and a discovered transformation is therefore worth retaining and
reusing.

\paragraph{Nonconvex stochastic and finite-sum optimization.}
Classical stochastic methods require $\mathcal O(1/\epsilon^4)$ oracle calls
for smooth nonconvex stationarity under bounded variance
\citep{ghadimi2013stochastic}. Variance-reduced methods improve component-gradient
complexity using finite-sum or recursive estimators
\citep{pmlr-v48-reddi16,pmlr-v70-nguyen17b,fang2018spider,
cutkosky2019momentum,pmlr-v139-li21a}. Random reshuffling provides another
improvement over independent sampling
\citep{mishchenko2020random,ahn2020sgd}, and oracle lower bounds characterize
the limits of nonconvex finite-sum and stochastic optimization
\citep{pmlr-v97-zhou19b,pmlr-v291-saad25a}. These works count stochastic- or
component-gradient evaluations. We instead count newly sampled
transformations. A transformation can be acquired once and reused throughout
optimization, which is the distinction exploited by fixed sparse
augmentation.

\section{Proofs and Technical Background}
\label{app:proofs}

This appendix provides the complete proof of the results stated in
Section~\ref{sec:main-results}. We first review the required background on
finite groups, unitary representations, and group-averaging operators. We
then prove the spectral approximation bound for a random sparse group
average, transfer this estimate to the augmented gradient fields through the
RKHS structure, and complete the convergence proof for one-shot sparse
gradient descent.

\subsection{Finite groups and group actions}
\label{app:groups}

A \emph{finite group} is a finite set $G$ equipped with a binary operation
$(g,h)\mapsto gh$ satisfying closure, associativity, the existence of an
identity element $e\in G$, and the existence of an inverse $g^{-1}\in G$ for
every $g\in G$.

A \emph{left action} of $G$ on a set $\mathcal X$ is a map
\begin{equation}
G\times\mathcal X\to\mathcal X,
\qquad
(g,x)\mapsto g\cdot x,
\end{equation}
such that
\begin{equation}
e\cdot x=x,
\qquad
(gh)\cdot x=g\cdot(h\cdot x)
\end{equation}
for every $g,h\in G$ and $x\in\mathcal X$. In particular, each $g\in G$
induces a bijection $x\mapsto g\cdot x$ whose inverse is
$x\mapsto g^{-1}\cdot x$.

Let $\mathcal H$ be a Hilbert space of functions
$h\colon\mathcal X\to\mathbb R^p$. The action of $G$ on $\mathcal X$ induces
operators $U_g\colon\mathcal H\to\mathcal H$ defined by
\begin{equation}
(U_gh)(x)
\coloneqq
h(g^{-1}\cdot x).
\label{eq:app-induced-action}
\end{equation}
These operators satisfy
\begin{equation}
U_e=I_{\mathcal H},
\qquad
U_gU_h=U_{gh}.
\end{equation}
Thus, $g\mapsto U_g$ is a representation of $G$ on $\mathcal H$.

Under Assumption~\ref{ass:gradient-rkhs}, this representation is unitary:
\begin{equation}
\langle U_gh,U_gh'\rangle_{\mathcal H}
=
\langle h,h'\rangle_{\mathcal H}
\qquad
\text{for all }g\in G
\text{ and }h,h'\in\mathcal H.
\label{eq:app-unitarity}
\end{equation}
Consequently,
\begin{equation}
U_g^*=U_g^{-1}=U_{g^{-1}},
\qquad
\|U_g\|_{\mathrm{op}}=1.
\label{eq:app-unitary-properties}
\end{equation}

\subsection{Representations of finite groups}
\label{app:representation-theory}

We briefly recall the representation-theoretic facts needed below. Although
$\mathcal H$ may be a real Hilbert space, it is convenient to work with its
complexification. All operator-norm estimates obtained after complexification
remain valid on the original real space.

A finite-dimensional complex representation
$\pi\colon G\to\mathcal U(V_\pi)$ is \emph{irreducible} if its only
$G$-invariant subspaces are $\{0\}$ and $V_\pi$. We write
$d_\pi\coloneqq\dim(V_\pi)$ and fix a complete collection $\widehat G$ of
pairwise inequivalent unitary irreducible representations of $G$, with each
$V_\pi$ equipped with a fixed $G$-invariant inner product.

Every unitary representation of a finite group decomposes into isotypic
components: there is a unitary isomorphism
\begin{equation}
W\colon\mathcal H
\longrightarrow
\bigoplus_{\pi\in\widehat G}
\mathcal M_\pi\otimes V_\pi,
\label{eq:app-isotypic-decomposition}
\end{equation}
where $\mathcal M_\pi$ is the Hilbert multiplicity space associated with
$\pi$. Under this unitary change of coordinates,
\begin{equation}
WU_gW^*
=
\bigoplus_{\pi\in\widehat G}
\left(I_{\mathcal M_\pi}\otimes\pi(g)\right).
\label{eq:app-block-representation}
\end{equation}
The multiplicity spaces carry the part of the Hilbert-space structure not
contained in the finite-dimensional $V_\pi$ and may themselves be
infinite-dimensional. Importantly, the action of $G$ is identical on every
copy of the same irreducible representation, so operator-norm control depends
on the distinct irreducible representations rather than their multiplicities.

The trivial representation, denoted $\mathbf 1$, is the one-dimensional
representation satisfying
\begin{equation}
\mathbf 1(g)=1
\qquad
\text{for every }g\in G.
\end{equation}
Its isotypic component is precisely the invariant subspace
\begin{equation}
\mathcal H^G
=
\{h\in\mathcal H:U_gh=h\text{ for every }g\in G\}.
\label{eq:app-invariant-subspace}
\end{equation}

We use two standard orthogonality identities. For any nontrivial
$\pi\in\widehat G$,
\begin{equation}
\frac{1}{|G|}\sum_{g\in G}\pi(g)=0,
\label{eq:app-nontrivial-average}
\end{equation}
whereas for the trivial representation,
\begin{equation}
\frac{1}{|G|}\sum_{g\in G}\mathbf 1(g)=1.
\end{equation}
We also use the dimension identity
\begin{equation}
\sum_{\pi\in\widehat G}d_\pi^2=|G|.
\label{eq:app-dimension-identity}
\end{equation}
In particular,
\begin{equation}
\sum_{\substack{\pi\in\widehat G\\\pi\neq\mathbf 1}}d_\pi
\leq
\sum_{\substack{\pi\in\widehat G\\\pi\neq\mathbf 1}}d_\pi^2
=
|G|-1
\leq |G|.
\label{eq:app-dimension-bound}
\end{equation}

\subsection{Full and sampled averaging operators}
\label{app:averaging-operators}

Define the full group-averaging operator
\begin{equation}
\Pi_G
\coloneqq
\frac{1}{|G|}\sum_{g\in G}U_g.
\label{eq:app-full-operator}
\end{equation}

\begin{lemma}[Full averaging is an orthogonal projection]
\label{lem:app-full-projection}
The operator $\Pi_G$ is the orthogonal projection from $\mathcal H$ onto
$\mathcal H^G$.
\end{lemma}

\begin{proof}
First, using the group property,
\begin{align}
\Pi_G^2
&=
\frac{1}{|G|^2}
\sum_{g,h\in G}U_{gh}.
\end{align}
For each $k\in G$, exactly $|G|$ ordered pairs $(g,h)$ satisfy $gh=k$.
Therefore,
\begin{equation}
\Pi_G^2
=
\frac{1}{|G|}
\sum_{k\in G}U_k
=
\Pi_G.
\end{equation}
Moreover, by \eqref{eq:app-unitary-properties} and invariance of the group
under inversion,
\begin{equation}
\Pi_G^*
=
\frac{1}{|G|}\sum_{g\in G}U_g^*
=
\frac{1}{|G|}\sum_{g\in G}U_{g^{-1}}
=
\Pi_G.
\end{equation}
Hence, $\Pi_G$ is an orthogonal projection.

For every $a\in G$,
\begin{equation}
U_a\Pi_G
=
\frac{1}{|G|}\sum_{g\in G}U_{ag}
=
\Pi_G,
\end{equation}
so $\operatorname{range}(\Pi_G)\subseteq\mathcal H^G$. Conversely, if
$h\in\mathcal H^G$, then
\begin{equation}
\Pi_Gh
=
\frac{1}{|G|}\sum_{g\in G}U_gh
=
h.
\end{equation}
Thus, $\operatorname{range}(\Pi_G)=\mathcal H^G$.
\end{proof}

Let $g_1,\ldots,g_m$ be independent uniform samples from $G$. Define
\begin{equation}
\Pi_S
\coloneqq
\frac{1}{m}\sum_{j=1}^mU_{g_j^{-1}}.
\label{eq:app-sampled-operator}
\end{equation}
Since
\begin{equation}
\mathbb E[U_{g_j^{-1}}]
=
\frac{1}{|G|}\sum_{g\in G}U_g
=
\Pi_G,
\end{equation}
the sampled operator is an unbiased estimator of $\Pi_G$:
\begin{equation}
\mathbb E[\Pi_S]=\Pi_G.
\end{equation}

Under the isotypic decomposition
\eqref{eq:app-isotypic-decomposition}, the two averaging operators take the
forms
\begin{align}
W\Pi_GW^*
&=
I_{\mathcal M_{\mathbf 1}}
\oplus
\bigoplus_{\substack{\pi\in\widehat G\\\pi\neq\mathbf 1}}
0_{\mathcal M_\pi\otimes V_\pi},
\label{eq:app-full-blocks}
\\
W\Pi_SW^*
&=
I_{\mathcal M_{\mathbf 1}}
\oplus
\bigoplus_{\substack{\pi\in\widehat G\\\pi\neq\mathbf 1}}
\left(
I_{\mathcal M_\pi}
\otimes
\frac{1}{m}\sum_{j=1}^m\pi(g_j^{-1})
\right).
\label{eq:app-sampled-blocks}
\end{align}
For an operator $A\colon V_\pi\to V_\pi$, let
$\|A\|_{\mathrm{op},V_\pi}
\coloneqq\sup_{v\neq0}\|Av\|_{V_\pi}/\|v\|_{V_\pi}$.
Unitary conjugation preserves operator norms, the norm of an orthogonal
direct sum is the supremum of its block norms, and
$\|I_{\mathcal M_\pi}\otimes A\|_{\mathrm{op}}=
\|A\|_{\mathrm{op},V_\pi}$ even when $\mathcal M_\pi$ is
infinite-dimensional. It follows that
\begin{equation}
\|\Pi_S-\Pi_G\|_{\mathrm{op},\mathcal H}
=
\sup_{\substack{\pi\in\widehat G\\
\pi\neq\mathbf 1,\ \mathcal M_\pi\neq\{0\}}}
\left\|
\frac{1}{m}\sum_{j=1}^m\pi(g_j^{-1})
\right\|_{\mathrm{op},V_\pi}.
\label{eq:app-operator-fourier}
\end{equation}
Thus, it suffices to control the empirical average of each nontrivial
irreducible representation. The Hilbert space $\mathcal H$ affects only which
irreducibles occur and their multiplicities; it does not alter the
finite-dimensional block matrices or their norms.

\subsection{Spectral concentration for sparse group averaging}
\label{app:tw-bound}

We now prove the spectral approximation result used in the main text. The
argument follows the random averaging construction of
\citep{tahmasebi2025achieving}, which is closely related to the
Alon--Roichman theorem for random Cayley graphs
\citep{alon1994random}.

We use the following matrix Bernstein inequality.

\begin{lemma}[Matrix Bernstein inequality \citep{tropp2012user}]
\label{lem:app-matrix-bernstein}
Let $X_1,\ldots,X_m$ be independent, mean-zero, complex
$d_1\times d_2$ random matrices satisfying
$\|X_j\|_{\mathrm{op}}\leq R$ almost surely. Define
\begin{equation}
v
\coloneqq
\max\left\{
\left\|\sum_{j=1}^m\mathbb E[X_jX_j^*]\right\|_{\mathrm{op}},
\left\|\sum_{j=1}^m\mathbb E[X_j^*X_j]\right\|_{\mathrm{op}}
\right\}.
\end{equation}
Then, for every $t\geq0$,
\begin{equation}
\mathbb P\left(
\left\|\sum_{j=1}^mX_j\right\|_{\mathrm{op}}\geq t
\right)
\leq
(d_1+d_2)
\exp\left(
-\frac{t^2/2}{v+Rt/3}
\right).
\label{eq:app-matrix-bernstein}
\end{equation}
\end{lemma}

\begin{theorem}[Random sparse approximation of group averaging]
\label{thm:app-tw}
Let $g_1,\ldots,g_m$ be independent uniform samples from a finite group $G$.
For every $0<\delta<1$, with probability at least $1-\delta$,
\begin{equation}
\|\Pi_S-\Pi_G\|_{\mathrm{op},\mathcal H}
\leq
\min\left\{
1,\,
\sqrt{
\frac{8}{3m}
\log\left(\frac{2|G|}{\delta}\right)
}
\right\},
\label{eq:app-tw-explicit}
\end{equation}
Moreover, this event depends only on $S$, and the bound holds for the
corresponding averaging operators on every possibly infinite-dimensional
Hilbert space carrying a unitary representation of $G$.
\end{theorem}

\begin{proof}
Fix a nontrivial irreducible unitary representation
$\pi\in\widehat G$ of dimension $d_\pi$. Define
\begin{equation}
X_j^{(\pi)}
\coloneqq
\pi(g_j^{-1}).
\end{equation}
By the orthogonality relation \eqref{eq:app-nontrivial-average},
\begin{equation}
\mathbb E[X_j^{(\pi)}]
=
\frac{1}{|G|}\sum_{g\in G}\pi(g^{-1})
=
0.
\end{equation}
Since $\pi$ is unitary,
\begin{equation}
\|X_j^{(\pi)}\|_{\mathrm{op},V_\pi}=1,
\qquad
X_j^{(\pi)}X_j^{(\pi)*}=I_{d_\pi},
\qquad
X_j^{(\pi)*}X_j^{(\pi)}=I_{d_\pi}.
\end{equation}
Thus, the matrix Bernstein parameters satisfy
\begin{equation}
R=1,
\qquad
v=m.
\end{equation}

Applying Lemma~\ref{lem:app-matrix-bernstein} with
$t=m\tau$, for $0<\tau\leq1$, gives
\begin{align}
\mathbb P\left(
\left\|
\frac{1}{m}\sum_{j=1}^m\pi(g_j^{-1})
\right\|_{\mathrm{op},V_\pi}
>\tau
\right)
&\leq
2d_\pi
\exp\left(
-\frac{m^2\tau^2/2}{m+m\tau/3}
\right)
\\
&\leq
2d_\pi
\exp\left(
-\frac{3m\tau^2}{8}
\right).
\label{eq:app-one-irrep-tail}
\end{align}
The final inequality uses $\tau\leq1$, and hence
$1+\tau/3\leq4/3$.

Taking a union bound over all nontrivial irreducible representations of $G$
yields
\begin{align}
&\mathbb P\left(
\sup_{\substack{\pi\in\widehat G\\\pi\neq\mathbf 1}}
\left\|
\frac{1}{m}\sum_{j=1}^m\pi(g_j^{-1})
\right\|_{\mathrm{op},V_\pi}
>\tau
\right)
\qquad\leq
2
\sum_{\substack{\pi\in\widehat G\\\pi\neq\mathbf 1}}
d_\pi
\exp\left(-\frac{3m\tau^2}{8}\right).
\end{align}
By \eqref{eq:app-dimension-bound},
\begin{equation}
\mathbb P\left(
\sup_{\substack{\pi\in\widehat G\\\pi\neq\mathbf 1}}
\left\|
\frac{1}{m}\sum_{j=1}^m\pi(g_j^{-1})
\right\|_{\mathrm{op},V_\pi}
>\tau
\right)
\leq
2|G|\exp\left(-\frac{3m\tau^2}{8}\right).
\label{eq:app-union-bound}
\end{equation}
Choosing
\begin{equation}
\tau
=
\sqrt{
\frac{8}{3m}
\log\left(\frac{2|G|}{\delta}\right)
}
\end{equation}
makes the right-hand side of \eqref{eq:app-union-bound} equal to $\delta$.
When this choice satisfies $\tau\leq1$, the conclusion follows from the
block identity \eqref{eq:app-operator-fourier}. If $\tau>1$, the same block
identity and the fact that each block is an average of unitary operators give
the deterministic bound
$\|\Pi_S-\Pi_G\|_{\mathrm{op},\mathcal H}\leq1$.
Because the event in \eqref{eq:app-union-bound} controls every
$\pi\in\widehat G$, it is independent of the choice of $\mathcal H$.
Equation~\eqref{eq:app-operator-fourier} therefore transfers the same event
to every unitary representation of $G$, regardless of its multiplicities.
\end{proof}

\begin{remark}[Why the dependence is logarithmic in $|G|$]
A direct matrix-concentration argument on the entire function space could
produce a factor depending on $\log\dim(\mathcal H)$ and would not apply
directly when $\mathcal H$ is infinite-dimensional. The
representation-theoretic decomposition removes the multiplicities of the
irreducible components. Only the distinct irreducible representations must
be controlled, and their dimensions satisfy
$\sum_{\pi\in\widehat G}d_\pi\leq |G|$. This is what produces the dependence
$\log|G|$.
\end{remark}

\begin{remark}[Connection with random Cayley graphs]
For the left regular representation on $\ell^2(G)$, the operator $\Pi_S$ is
the normalized convolution operator associated with the empirical measure
of $S^{-1}$. If the multiset is symmetrized by including inverses, this operator
is the normalized adjacency operator of an undirected random Cayley
multigraph. The bound in Theorem~\ref{thm:app-tw} controls the nontrivial
spectrum of this graph. The formulation above does not require
symmetrization because matrix Bernstein applies directly to the possibly
non-self-adjoint representation blocks.
\end{remark}

\subsection{Vector-valued RKHS preliminaries}
\label{app:rkhs}

Let $\mathcal H$ be a vector-valued RKHS of functions
$h\colon\mathcal X\to\mathbb R^p$ with operator-valued reproducing kernel
\begin{equation}
K\colon\mathcal X\times\mathcal X
\to\mathbb R^{p\times p}.
\end{equation}
The reproducing property states that
\begin{equation}
\langle h(x),v\rangle_{\mathbb R^p}
=
\langle h,K(\,\cdot\,,x)v\rangle_{\mathcal H}
\label{eq:app-reproducing-property}
\end{equation}
for every $h\in\mathcal H$, $x\in\mathcal X$, and $v\in\mathbb R^p$.

\begin{lemma}[Uniform point-evaluation bound]
\label{lem:app-evaluation}
Suppose
\begin{equation}
\sup_{x\in\mathcal X}
\|K(x,x)\|_{\mathrm{op}}^{1/2}
\leq
C_{\mathcal H}.
\end{equation}
Then
\begin{equation}
\|h\|_{L^\infty}
\leq
C_{\mathcal H}\|h\|_{\mathcal H}
\qquad
\text{for every }h\in\mathcal H.
\label{eq:app-evaluation-bound}
\end{equation}
\end{lemma}

\begin{proof}
Fix $x\in\mathcal X$. For every $v\in\mathbb R^p$ with $\|v\|=1$, the
reproducing property and the Cauchy--Schwarz inequality give
\begin{align}
|\langle h(x),v\rangle|
&=
|\langle h,K(\,\cdot\,,x)v\rangle_{\mathcal H}|
\\
&\leq
\|h\|_{\mathcal H}
\|K(\,\cdot\,,x)v\|_{\mathcal H}.
\end{align}
Applying the reproducing property once more,
\begin{equation}
\|K(\,\cdot\,,x)v\|_{\mathcal H}^2
=
\langle K(x,x)v,v\rangle
\leq
C_{\mathcal H}^2.
\end{equation}
Taking the supremum over all unit vectors $v$ gives
\begin{equation}
\|h(x)\|
\leq
C_{\mathcal H}\|h\|_{\mathcal H}.
\end{equation}
Finally, take the supremum over $x\in\mathcal X$.
\end{proof}

\begin{remark}[Invariant kernels]
A sufficient condition for the pullback action
$(U_gh)(x)=h(g^{-1}\cdot x)$ to be unitary is
\begin{equation}
K(g\cdot x,g\cdot x')
=
K(x,x')
\qquad
\text{for all }g\in G
\text{ and }x,x'\in\mathcal X.
\end{equation}
As a nontrivial matrix-valued example on $\mathcal X=\mathbb R^d$ (or any
subset of that), consider the normalized diagonal \emph{spectral-mixture
kernel}
\begin{equation}
K(x,x')
=
\operatorname{diag}\bigl(
k_1(x-x'),\ldots,k_p(x-x')
\bigr),
\end{equation}
where, for $r\in\{1,\ldots,p\}$,
\begin{equation}
k_r(z)
=
\sum_{a=1}^{A_r}
\gamma_{r,a}
\exp\left(
-\frac{1}{2}z^\top\Lambda_{r,a}^{-1}z
\right)
\cos\left(\omega_{r,a}^\top z\right),
\end{equation}
with scalar weights and vector frequencies satisfying
\begin{equation}
\begin{gathered}
\gamma_{r,a}\geq0,
\qquad
\sum_{a=1}^{A_r}\gamma_{r,a}=1,\\
\omega_{r,a}\in\mathbb R^d,
\qquad
\Lambda_{r,a}\succ0.
\end{gathered}
\end{equation}
Different output coordinates may use different frequencies
$\omega_{r,a}$, bandwidth matrices $\Lambda_{r,a}$, and numbers of mixture
components, so $K$ need not be a scalar multiple of $I_p$. If
$g\cdot x=Q_gx+a_g$ is a Euclidean isometry and, for every $r,a,g$,
\begin{equation}
Q_g\Lambda_{r,a}Q_g^\top=\Lambda_{r,a},
\qquad
Q_g^\top\omega_{r,a}\in
\{\omega_{r,a},-\omega_{r,a}\},
\end{equation}
then each $k_r$ and hence $K$ is $G$-invariant. Moreover,
$K(x,x)=I_p$ for every $x$, so the associated vector-valued RKHS satisfies
$C_{\mathcal H}=1$ throughout this family, without any compactness assumption
on $\mathcal X$. The Gaussian kernel is recovered as the special case
$A_r=1$, $\gamma_{r,1}=1$, $\omega_{r,1}=0$, and
$\Lambda_{r,1}=\rho^2I_d$ for every $r$.
\end{remark}

\subsection{Uniform approximation of the augmented gradients}
\label{app:gradient-transfer}

For every $i$ and $\theta$, recall the gradient function
\begin{equation}
h_{i,\theta}(x)
=
\nabla_\theta\ell_i(x;\theta).
\end{equation}
Differentiating the augmented risks gives
\begin{align}
\nabla\mathcal R_n^G(\theta)
&=
\frac{1}{n|G|}
\sum_{i=1}^n
\sum_{g\in G}
h_{i,\theta}(g\cdot x_i),
\label{eq:app-full-gradient}
\\
\nabla\mathcal R_n^S(\theta)
&=
\frac{1}{nm}
\sum_{i=1}^n
\sum_{j=1}^m
h_{i,\theta}(g_j\cdot x_i).
\label{eq:app-sparse-gradient}
\end{align}

By the augmentation-aligned definition of $\Pi_S$ in
\eqref{eq:app-sampled-operator},
\begin{align}
(\Pi_Sh)(x)
&=
\frac{1}{m}\sum_{j=1}^m
h(g_j\cdot x),                                      \\
(\Pi_Gh)(x)
&=
\frac{1}{|G|}\sum_{g\in G}h(g\cdot x),
\end{align}
where the second identity uses the fact that inversion permutes $G$.
Consequently,
\begin{equation}
\nabla\mathcal R_n^S(\theta)
-
\nabla\mathcal R_n^G(\theta)
=
\frac{1}{n}
\sum_{i=1}^n
\left[
(\Pi_S-\Pi_G)h_{i,\theta}
\right](x_i).
\label{eq:app-gradient-operator-identity}
\end{equation}

\begin{theorem}[Uniform sparse-to-full gradient approximation]
\label{thm:app-uniform-gradient}
Suppose Assumption~\ref{ass:gradient-rkhs} holds. Then, with probability at
least $1-\delta$,
\begin{equation}
\sup_{\theta\in\Theta}
\left\|
\nabla\mathcal R_n^S(\theta)
-
\nabla\mathcal R_n^G(\theta)
\right\|
\leq
C_{\mathcal H}B_{\mathcal H}
\sqrt{
\frac{8}{3m}
\log\left(\frac{2|G|}{\delta}\right)
}.
\label{eq:app-uniform-gradient-bound}
\end{equation}
\end{theorem}

\begin{proof}
Condition on the event in Theorem~\ref{thm:app-tw}. For any fixed
$\theta\in\Theta$, \eqref{eq:app-gradient-operator-identity}, the triangle
inequality, and Lemma~\ref{lem:app-evaluation} give
\begin{align}
\left\|
\nabla\mathcal R_n^S(\theta)
-
\nabla\mathcal R_n^G(\theta)
\right\| 
& \leq
\frac{1}{n}
\sum_{i=1}^n
\left\|
\left[
(\Pi_S-\Pi_G)h_{i,\theta}
\right](x_i)
\right\|
\\
&\quad\leq
\frac{C_{\mathcal H}}{n}
\sum_{i=1}^n
\left\|
(\Pi_S-\Pi_G)h_{i,\theta}
\right\|_{\mathcal H}
\\
&\quad\leq
C_{\mathcal H}
\|\Pi_S-\Pi_G\|_{\mathrm{op},\mathcal H}
 \times  \frac{1}{n}
\sum_{i=1}^n
\|h_{i,\theta}\|_{\mathcal H}.
\end{align}
Assumption~\ref{ass:gradient-rkhs} implies
\begin{equation}
\frac{1}{n}
\sum_{i=1}^n
\|h_{i,\theta}\|_{\mathcal H}
\leq
B_{\mathcal H}
\end{equation}
uniformly over $\theta$. Therefore,
\begin{equation}
\left\|
\nabla\mathcal R_n^S(\theta)
-
\nabla\mathcal R_n^G(\theta)
\right\|
\leq
C_{\mathcal H}B_{\mathcal H}
\sqrt{
\frac{8}{3m}
\log\left(\frac{2|G|}{\delta}\right)
}
\end{equation}
simultaneously for all $\theta\in\Theta$.
\end{proof}

\begin{remark}[Why pointwise concentration is insufficient]
For a fixed $\theta$ chosen independently of $S$, ordinary concentration
could control
$\nabla\mathcal R_n^S(\theta)-\nabla\mathcal R_n^G(\theta)$. However, the
iterates of one-shot sparse GD depend on $S$, so such a pointwise statement
cannot be substituted directly at $\theta=\theta_t$. The operator event in
Theorem~\ref{thm:app-tw} holds for every $h\in\mathcal H$ simultaneously,
and the RKHS norm bound is uniform over $\theta$. Thus,
\eqref{eq:app-uniform-gradient-bound} remains valid along the entire
$S$-dependent trajectory without requiring independence between $S$ and
$\theta_t$.
\end{remark}

\subsection{Nonconvex gradient descent}
\label{app:gd-proof}

We next recall the standard descent estimate for a smooth nonconvex
objective.

\begin{lemma}[Descent lemma]
\label{lem:app-descent}
Let $F\colon\mathbb R^p\to\mathbb R$ have an $L$-Lipschitz gradient. Then,
for every $\theta,\theta'\in\mathbb R^p$,
\begin{equation}
F(\theta')
\leq
F(\theta)
+
\langle\nabla F(\theta),\theta'-\theta\rangle
+
\frac{L}{2}\|\theta'-\theta\|^2.
\label{eq:app-descent-lemma}
\end{equation}
\end{lemma}

\begin{proof}
Define the line segment
$\gamma(s)\coloneqq\theta+s(\theta'-\theta)$ for $s\in[0,1]$. By the
fundamental theorem of calculus,
\begin{align}
F(\theta')-F(\theta)
&=
\int_0^1
\langle \nabla F(\gamma(s)),\theta'-\theta\rangle\,\mathrm ds
\\
&=
\langle\nabla F(\theta),\theta'-\theta\rangle
+
\int_0^1
\langle\nabla F(\gamma(s))-\nabla F(\theta),
\theta'-\theta\rangle\,\mathrm ds.
\end{align}
The Cauchy--Schwarz inequality and the $L$-Lipschitz continuity of
$\nabla F$ imply
\begin{align}
\int_0^1
\langle\nabla F(\gamma(s))-\nabla F(\theta),
\theta'-\theta\rangle\,\mathrm ds
&\leq
\int_0^1
\|\nabla F(\gamma(s))-\nabla F(\theta)\|
\,\|\theta'-\theta\|\,\mathrm ds
\\
&\leq
\int_0^1 Ls\|\theta'-\theta\|^2\,\mathrm ds
=
\frac{L}{2}\|\theta'-\theta\|^2.
\end{align}
Combining the preceding two displays gives
\eqref{eq:app-descent-lemma}.
\end{proof}

\begin{lemma}[Nonconvex GD rate]
\label{lem:app-gd-rate}
Suppose $F$ is $L$-smooth and bounded below. Let
\begin{equation}
\theta_{t+1}
=
\theta_t-\eta\nabla F(\theta_t),
\qquad
0<\eta\leq\frac{1}{L}.
\end{equation}
Then
\begin{equation}
\frac{1}{T}
\sum_{t=0}^{T-1}
\|\nabla F(\theta_t)\|^2
\leq
\frac{2(F(\theta_0)-\inf_\theta F(\theta))}{\eta T}.
\label{eq:app-gd-average}
\end{equation}
Consequently,
\begin{equation}
\min_{0\leq t<T}
\|\nabla F(\theta_t)\|
\leq
\sqrt{
\frac{2(F(\theta_0)-\inf_\theta F(\theta))}{\eta T}
}.
\label{eq:app-gd-minimum}
\end{equation}
\end{lemma}

\begin{proof}
Apply Lemma~\ref{lem:app-descent} with
$\theta'=\theta-\eta\nabla F(\theta)$. This gives
\begin{align}
F(\theta_{t+1})
&\leq
F(\theta_t)
-
\eta\|\nabla F(\theta_t)\|^2
+
\frac{L\eta^2}{2}\|\nabla F(\theta_t)\|^2
\\
&=
F(\theta_t)
-
\eta\left(1-\frac{L\eta}{2}\right)
\|\nabla F(\theta_t)\|^2.
\end{align}
Since $\eta\leq1/L$,
\begin{equation}
F(\theta_{t+1})
\leq
F(\theta_t)
-
\frac{\eta}{2}
\|\nabla F(\theta_t)\|^2.
\end{equation}
Summing over $t=0,\ldots,T-1$ yields
\begin{equation}
\frac{\eta}{2}
\sum_{t=0}^{T-1}
\|\nabla F(\theta_t)\|^2
\leq
F(\theta_0)-F(\theta_T)
\leq
F(\theta_0)-\inf_\theta F(\theta).
\end{equation}
Dividing by $\eta T/2$ proves \eqref{eq:app-gd-average}. The minimum is no
larger than the average, which proves \eqref{eq:app-gd-minimum}.
\end{proof}

\subsection{Proof of the main theorem}
\label{app:main-proof}

We now combine the uniform gradient approximation with the descent estimate.

\begin{theorem}[One-shot sparse GD]
\label{thm:app-main}
Suppose Assumptions~\ref{ass:gradient-rkhs}
and~\ref{ass:initial-gap} hold. Run
Algorithm~\ref{alg:sparse-aug} with $\eta_t=1/L$ for every $t$, and let
\begin{equation}
\widehat t
\in
\operatorname*{arg\,min}_{0\leq t<T}
\|\nabla\mathcal R_n^S(\theta_t)\|,
\qquad
\widehat\theta=\theta_{\widehat t}.
\end{equation}
Then, with probability at least $1-\delta$,
\begin{equation}
\|\nabla\mathcal R_n^G(\widehat\theta)\|
\leq
\sqrt{\frac{2L\Delta}{T}}
+
C_{\mathcal H}B_{\mathcal H}
\sqrt{
\frac{8}{3m}
\log\left(\frac{2|G|}{\delta}\right)
}.
\label{eq:app-main-bound}
\end{equation}
\end{theorem}

\begin{proof}
Apply Lemma~\ref{lem:app-gd-rate} to
$F=\mathcal R_n^S$ with $\eta=1/L$. Condition~\eqref{eq:smoothness}
ensures that $F$ is $L$-smooth, while
Assumption~\ref{ass:initial-gap} gives
\begin{equation}
\mathcal R_n^S(\theta_0)
-
\inf_{\theta}\mathcal R_n^S(\theta)
\leq
\Delta.
\end{equation}
Therefore,
\begin{equation}
\|\nabla\mathcal R_n^S(\widehat\theta)\|
\leq
\sqrt{\frac{2L\Delta}{T}}.
\label{eq:app-sparse-stationarity}
\end{equation}

On the event of Theorem~\ref{thm:app-uniform-gradient}, which has probability
at least $1-\delta$,
\begin{align}
\|\nabla\mathcal R_n^G(\widehat\theta)\|
&\leq
\|\nabla\mathcal R_n^S(\widehat\theta)\|
+
\|\nabla\mathcal R_n^G(\widehat\theta)
-\nabla\mathcal R_n^S(\widehat\theta)\|
\\
&\leq
\sqrt{\frac{2L\Delta}{T}}
+
C_{\mathcal H}B_{\mathcal H}
\sqrt{
\frac{8}{3m}
\log\left(\frac{2|G|}{\delta}\right)
}.
\end{align}
This is precisely \eqref{eq:app-main-bound}.
\end{proof}

\begin{corollary}[Stationarity and group-oracle complexity]
\label{cor:app-main-complexity}
Under the assumptions of Theorem~\ref{thm:app-main}, suppose
\begin{equation}
T
\geq
\frac{8L\Delta}{\epsilon^2}
\label{eq:app-T-choice}
\end{equation}
and
\begin{equation}
m
\geq
\frac{
32C_{\mathcal H}^2B_{\mathcal H}^2
}{
3\epsilon^2
}
\log\left(\frac{2|G|}{\delta}\right).
\label{eq:app-m-choice}
\end{equation}
Then
\begin{equation}
\|\nabla\mathcal R_n^G(\widehat\theta)\|
\leq
\epsilon
\end{equation}
with probability at least $1-\delta$.
\end{corollary}

\begin{proof}
The choice \eqref{eq:app-T-choice} gives
\begin{equation}
\sqrt{\frac{2L\Delta}{T}}
\leq
\frac{\epsilon}{2}.
\end{equation}
Similarly, \eqref{eq:app-m-choice} gives
\begin{equation}
C_{\mathcal H}B_{\mathcal H}
\sqrt{
\frac{8}{3m}
\log\left(\frac{2|G|}{\delta}\right)
}
\leq
\frac{\epsilon}{2}.
\end{equation}
The result follows from Theorem~\ref{thm:app-main}.
\end{proof}

Because the sample $S$ is acquired only once, the number of group-oracle calls
is exactly $m$. Therefore,
\begin{equation}
m
=
\mathcal O\left(
\frac{
C_{\mathcal H}^2B_{\mathcal H}^2
\bigl(\log|G|+\log(1/\delta)\bigr)
}{
\epsilon^2
}
\right).
\label{eq:app-final-oracle-complexity}
\end{equation}
The number of optimization iterations is
\begin{equation}
T
=
\mathcal O\left(
\frac{L\Delta}{\epsilon^2}
\right).
\label{eq:app-final-iteration-complexity}
\end{equation}
The sampled transformations are cached after the initial oracle calls and can
be reused at every iteration without any further transformation acquisition.

\subsection{Classical baseline rates}
\label{app:baseline-proofs}

For completeness, we state the standard full-group GD and group-SGD
guarantees used for comparison.

\subsubsection{Full-group gradient descent}

Applying Lemma~\ref{lem:app-gd-rate} directly to
$F=\mathcal R_n^G$ with $\eta=1/L$ gives
\begin{equation}
\min_{0\leq t<T}
\|\nabla\mathcal R_n^G(\theta_t)\|^2
\leq
\frac{2L\Delta_G}{T},
\label{eq:app-full-gd-rate}
\end{equation}
where
\begin{equation}
\Delta_G
\coloneqq
\mathcal R_n^G(\theta_0)
-
\inf_\theta\mathcal R_n^G(\theta).
\end{equation}
Thus, $T\geq2L\Delta_G/\epsilon^2$ iterations suffice for
$\epsilon$-stationarity. This method requires access to all $|G|$ group
elements, corresponding to full-group oracle access.

\subsubsection{Streaming group-SGD}

For each iteration $t$, let
$S_t=\{g_{t,1},\ldots,g_{t,b}\}$ be a fresh batch of $b$ independent uniform
samples from $G$, and define
\begin{equation}
r_g(\theta)
\coloneqq
\frac{1}{n}
\sum_{i=1}^n
\ell_i(g\cdot x_i;\theta).
\end{equation}
Then
\begin{equation}
\mathcal R_n^G(\theta)
=
\mathbb E_g[r_g(\theta)]
\end{equation}
and
\begin{equation}
\mathbb E_g[\nabla r_g(\theta)]
=
\nabla\mathcal R_n^G(\theta).
\end{equation}
Assume that
\begin{equation}
\mathbb E_g
\left[
\|\nabla r_g(\theta)-\nabla\mathcal R_n^G(\theta)\|^2
\right]
\leq
\sigma_G^2
\label{eq:app-group-variance}
\end{equation}
uniformly over $\theta$. Writing
\begin{equation}
\nabla\mathcal R_n^{S_t}(\theta)
=
\frac{1}{b}
\sum_{j=1}^b \nabla r_{g_{t,j}}(\theta),
\end{equation}
independence within the batch gives
\begin{equation}
\mathbb E
\left[
\left\|
\nabla\mathcal R_n^{S_t}(\theta)
-
\nabla\mathcal R_n^G(\theta)
\right\|^2
\right]
\leq
\frac{\sigma_G^2}{b}.
\label{eq:app-batch-variance}
\end{equation}
Group-SGD performs
\begin{equation}
\theta_{t+1}
=
\theta_t-\eta\nabla\mathcal R_n^{S_t}(\theta_t).
\end{equation}

\begin{proposition}[Streaming group-SGD]
\label{prop:app-group-sgd}
Let $\tau$ be independent and uniformly distributed over
$\{0,\ldots,T-1\}$. If $0<\eta\leq1/L$, then
\begin{equation}
\mathbb E
\left[
\|\nabla\mathcal R_n^G(\theta_\tau)\|^2
\right]
\leq
\frac{2\Delta_G}{\eta T}
+
\frac{L\eta\sigma_G^2}{b}.
\label{eq:app-sgd-rate}
\end{equation}
Here, as in the preceding subsection,
$\Delta_G=\mathcal R_n^G(\theta_0)-\inf_\theta\mathcal R_n^G(\theta)$.
\end{proposition}

\begin{proof}
Let $\mathcal F_t$ denote the history before the fresh batch $S_t$ is drawn,
and write $\mathbb E_t[\,\cdot\,]
\coloneqq\mathbb E[\,\cdot\mid\mathcal F_t]$. In particular, $\theta_t$ is
$\mathcal F_t$-measurable. Conditioning on $\mathcal F_t$ and applying the
descent lemma gives
\begin{align}
\mathbb E_t[\mathcal R_n^G(\theta_{t+1})]
&\leq
\mathcal R_n^G(\theta_t)
-
\eta
\left\langle
\nabla\mathcal R_n^G(\theta_t),
\mathbb E_t[\nabla\mathcal R_n^{S_t}(\theta_t)]
\right\rangle
\nonumber\\
&\quad+
\frac{L\eta^2}{2}
\mathbb E_t
\left[
\|\nabla\mathcal R_n^{S_t}(\theta_t)\|^2
\right].
\end{align}
Conditional unbiasedness and \eqref{eq:app-batch-variance} imply
\begin{align}
\mathbb E_t
\left[
\|\nabla\mathcal R_n^{S_t}(\theta_t)\|^2
\right]
&=
\|\nabla\mathcal R_n^G(\theta_t)\|^2
+
\mathbb E_t
\left[
\|\nabla\mathcal R_n^{S_t}(\theta_t)
-\nabla\mathcal R_n^G(\theta_t)\|^2
\right]
\nonumber\\
&\leq
\|\nabla\mathcal R_n^G(\theta_t)\|^2+
\frac{\sigma_G^2}{b}.
\end{align}
The equality is the conditional variance decomposition; conditional
unbiasedness makes its cross term zero.
Therefore,
\begin{align}
\mathbb E_t[\mathcal R_n^G(\theta_{t+1})]
&\leq
\mathcal R_n^G(\theta_t)
-
\eta\left(1-\frac{L\eta}{2}\right)
\|\nabla\mathcal R_n^G(\theta_t)\|^2
+
\frac{L\eta^2\sigma_G^2}{2b}
\\
&\leq
\mathcal R_n^G(\theta_t)
-
\frac{\eta}{2}
\|\nabla\mathcal R_n^G(\theta_t)\|^2
+
\frac{L\eta^2\sigma_G^2}{2b}.
\end{align}
Now take total expectations. The tower property gives
$\mathbb E[\mathbb E_t[\mathcal R_n^G(\theta_{t+1})]]
=\mathbb E[\mathcal R_n^G(\theta_{t+1})]$; the remaining terms are already
$\mathcal F_t$-measurable. Summing the resulting inequality over $t$ gives
\begin{equation}
\frac{\eta}{2}
\sum_{t=0}^{T-1}
\mathbb E
\left[
\|\nabla\mathcal R_n^G(\theta_t)\|^2
\right]
\leq
\Delta_G
+
\frac{L\eta^2\sigma_G^2T}{2b}.
\end{equation}
Divide by $\eta T/2$ and use the uniform distribution of $\tau$.
\end{proof}

Choosing
\begin{equation}
\eta
=
\min\left\{
\frac{1}{L},
\frac{b\epsilon^2}{2L\sigma_G^2}
\right\}
\end{equation}
and
\begin{equation}
T
\geq
\frac{4\Delta_G}{\eta\epsilon^2}
\end{equation}
ensures
\begin{equation}
\mathbb E
\left[
\|\nabla\mathcal R_n^G(\theta_\tau)\|^2
\right]
\leq
\epsilon^2.
\end{equation}
Consequently,
\begin{equation}
T
=
\mathcal O\left(
\frac{L\Delta_G}{\epsilon^2}
+
\frac{L\Delta_G\sigma_G^2}{b\epsilon^4}
\right).
\label{eq:app-sgd-complexity}
\end{equation}
Thus, when the stochastic-variance term dominates, the iteration complexity
scales as $\mathcal O(1/(b\epsilon^4))$, as reported in
Table~\ref{tab:nonconvex-comparison}. Since each iteration draws $b$ fresh
transformations, the resulting group-oracle complexity is
$\mathcal O(1/\epsilon^4)$ and does not improve with the batch size.

\end{document}